\documentclass{article}

\usepackage{microtype}
\usepackage{graphicx}
\usepackage{subfigure}
\usepackage{booktabs} 
\usepackage{tablefootnote}

\usepackage{hyperref}

\usepackage{amsmath,amsfonts,amssymb,latexsym,epsfig}
\usepackage{mathrsfs,mdframed, float,graphicx,dsfont,wrapfig}
\usepackage{amsthm}
\usepackage{graphics}
\usepackage{mathtools}  
\mathtoolsset{showonlyrefs=true}  

\newtheorem{theorem}{Theorem}[section]
\newtheorem{lemma}[theorem]{Lemma}
\newtheorem{remark}[theorem]{Remark}

\newtheorem{corollary}[theorem]{Corollary}

\newtheorem{assumption}{Assumption}
\newtheorem{algorithm}{Algorithm}
\newtheorem{proposition}{Proposition}

\newcommand{\bx}{\mathbf x}

\newcommand{\norm}[1]{\left\lVert#1\right\rVert}

\newcommand{\bSigma}{\boldsymbol{\Sigma}}

\definecolor{brinkpink}{rgb}{0.98, 0.38, 0.5}
\definecolor{blue_dark}{rgb}{0.0, 0.5, 0.69}

\newcommand{\eg}{\emph{e.g.}\;}
\newcommand{\ie}{\emph{i.e.}\;}

\usepackage[accepted]{icml2023}

\icmltitlerunning{Privacy Risk for anisotropic Langevin dynamics using relative entropy bounds}

\begin{document}

\twocolumn[
\icmltitle{Privacy Risk for anisotropic Langevin dynamics using relative entropy bounds}



\icmlsetsymbol{equal}{*}

\begin{icmlauthorlist}
\icmlauthor{Anastasia Borovykh}{1}
\icmlauthor{Nikolas Kantas}{1}
\icmlauthor{Panos Parpas}{2}
\icmlauthor{Greg Pavliotis}{1}
\end{icmlauthorlist}

\icmlaffiliation{1}{Department of Mathematics, Imperial College London, UK}
\icmlaffiliation{2}{Department of Computing, Imperial College London, UK}

\icmlcorrespondingauthor{Anastasia Borovykh}{a.borovykh@imperial.ac.uk}

\icmlkeywords{Langevin dynamics, anisotropic noise, differential privacy}

\vskip 0.3in
]



\printAffiliationsAndNotice{\icmlEqualContribution} 

\begin{abstract}
The privacy preserving properties of Langevin dynamics with additive isotropic noise have been extensively studied. However, the isotropic noise assumption is very restrictive: (a) when adding noise to existing learning algorithms to preserve privacy and maintain the best possible accuracy one should take into account the relative magnitude of the outputs and their correlations; (b) popular algorithms such as stochastic gradient descent (and their continuous time limits) appear to possess anisotropic covariance properties. To study the privacy risks for the anisotropic noise case, one requires general results on the relative entropy between the laws of two Stochastic Differential Equations with different drifts and diffusion coefficients. Our main contribution is to establish such a bound using stability estimates for solutions to the Fokker-Planck equations via functional inequalities. With additional assumptions, the relative entropy bound implies an $(\epsilon,\delta)$-differential privacy bound or translates to bounds on the membership inference attack success and we show how anisotropic noise can lead to better privacy-accuracy trade-offs. 
Finally, the benefits of anisotropic noise are illustrated using numerical results in quadratic loss and neural network setups. 
\end{abstract}
\section{Introduction}
In recent years a plethora of attacks have been defined which successfully extract private information from a trained model \cite{shokri2017membership, melis2019exploiting}. One example is the membership attack \cite{shokri2017membership, jagielski2020auditing} which aims to identify if a particular datapoint was part of the training dataset or not. This leaves individuals who contributed to the training data exposed to privacy threats. A framework for formally quantifying such privacy risks is through the concept of  $(\epsilon,\delta)$ differential privacy (DP) \cite{dwork2014algorithmic}. DP provides a bound on the probability that the output of the algorithm is able to reveal information about individual training points. Understanding the privacy risks of optimisation algorithms used for training ML models is essential for the models' safe application to real life settings. 
\subsection{Motivation}
In this work we will consider the following It\^{o} stochastic differential equation (SDE) dynamics,
\begin{align}\label{eq:SGD_dynamics}
    d\bx_t = -\nabla f(\bx_t)dt + \Sigma(\bx_t)dW_t.
\end{align}
We will be primarily interested in the case where the covariance matrix is anisotropic, \ie when $\Sigma\neq \sigma I$ for some constant $\sigma>0$. These dynamics can be seen as a continuous-time approximation of stochastic gradient descent (SGD) \cite{li-noise-sgd, chaudhari18, wu2020noisy} with $\Sigma$ being anisotropic due to the subsampling when mini-batching. It has been demonstrated numerically in \cite{hyland2019empirical} that this type of intrinsic (and anisotropic) noise in SGD can be leveraged for improving privacy risk. Alternatively, \eqref{eq:SGD_dynamics} can be also used to model cases where noise is added incrementally between iterations of a gradient learning algorithm for the purpose of privacy \cite{dwork2014algorithmic}, for example using Gaussian noise as in the popular Gaussian Mechanism (GM) \cite{balle18a-gaussian-mech}. 
Numerical case studies have shown that when adding noise with the same variance in each coordinate (or neural network layer), it can be hard to tune the variance such that good privacy-accuracy tradeoffs are achieved, e.g. see \cite{abadi2016deep, mo2021quantifying} for examples related to the accuracy of the model and their sensitivity to information extraction attacks. Relatedly, \cite{hall2013differential, chanyaswad2018mvg,yang2021improved, smith2018differentially}
have shown that 
anisotropic noise may result in better accuracy-privacy trade-offs. 
In our setting the accuracy-privacy tradeoff is influenced by the drift of the SDE, so that the loss function plays a role in the impact of anisotropic noise. 

\subsection{Contribution}

Our main result is a general bound on the relative entropy between the laws of two SDEs with different drifts and diffusion coefficients. We will include the possibility of both anisotropic and multiplicative noise. The result provides theoretical justification for injecting \emph{anisotropic} noise into \eqref{eq:SGD_dynamics}: minimising the relative entropy reduces privacy risk and the use of anisotropic noise can improve the overall accuracy at a similar privacy risk level, e.g. by adding noise only in the most critical directions (those with the highest magnitude of the gradient differences). 
We discuss the specific impact of our bound on various privacy risks in Section \ref{sec:insights}. First, the relative entropy bounds the membership inference attack advantage, hence choosing the right noise structure will allow to decrease the attacker's success while preserving accuracy. Second, if the log-Sobolev inequality (LSI) holds for the reference law, concentration of measure results can be obtained~\cite{ledoux1999concentration} and this directly translates the relative entropy bound to an $(\epsilon,\delta)$-DP bound. 
Third, it allows to quantify the impact of the batch size and loss structure on the priavcy risk.  
We end our analysis with numerical explorations of the privacy risks in quadratic and neural network models. 

\noindent
\textbf{Related work:} \cite{chourasia2021differential,altschuler2022privacy} derived time-dependent DP bounds for Langevin dynamics with \emph{isotropic} noise and showed that under certain assumptions the privacy risk converges over time. \cite{ganesh2022langevin} uses uniform stability to obtain tight DP bounds. Other work \cite{minami2015differential, minami2016differential, bogachev2014kantorovich} has considered the privacy risk for the Gibbs measure $\frac{1}{Z}e^{-f}$ and \cite{ganesh2020faster} extend the guarantees to dynamics that approximately sample from the Gibbs measure. An important technical challenge for the anisotropic noise case ($\Sigma\neq \sigma^2 I$ ) is that the SDE in \eqref{eq:SGD_dynamics} is not necessarily ergodic and not reversible with respect to a Gibbs measure $\frac{1}{Z}e^{- \frac{f}{2\sigma^2}}$ \cite{pavliotis14book}[Ch. 4], \cite{duncan2016variance} and different bounds are required. 

To address the challenge outlined above, one could use functional analysis inequalities to obtain bounds on the distance between the densities of two general SDEs at any point in time. \cite{manita2015estimates, bogachev2016distances, eberle2019sticky, wibisono2017information} and \cite{bogachev2014kantorovich, bogachev2018poisson} bound the distance between (non-Gibbs type) invariant measures of SDEs with different drifts. Relatedly, the work of \cite{assaraf2018computation} analyses the stability of a diffusion process with respect to drift perturbations. 
Our main technical contribution consists of extending and modifying  the techniques from \cite{bogachev2016distances} and stability estimates for solutions to Fokker-Planck equations via functional inequalities, so that they become relevant and applicable in the study of privacy-preserving algorithms.  


\noindent
\textbf{Notation and conventions:} We will use the following notation: let $A(x):\mathbb{R}^d\rightarrow \mathbb{R}^{m\times m}$ the divergence is defined as $\nabla\cdot A = \left[\sum_{k=1}^m \partial_{x_k}A_{1k},..., \sum_{k=1}^m \partial_{x_k}A_{mk}\right]^T$, the Frobenius product will be $A: B = \sum_{i}\sum_{j}A_{ij}B_{ij}=Tr(A^TB)$ and we will denote by $H f$ the Hessian matrix of a scalar function $f$, $H f :=D^2 f = \left\{\frac{\partial^2 f }{\partial x_i \partial x_j} \right\}_{i,j=1}^d $.

The relative entropy, or Kullback-Leibler (KL) divergence, between the two absolutely continuous probability measures $p\bx)d\bx$ and $p'(\bx)d\bx$, is given by:
$$KL(p||p') =\int_{\mathbb{R}^d}\log\frac{p(\bx)}{p'(\bx)}p(\bx) d\bx.$$

Let $\nu$ be a probability measure. Define the entropy
$$
\mbox{Ent}_{\nu}(g^2) = \mathbb{E}_\nu[g^2\log g^2]-\mathbb{E}_\nu[g^2]\ln\mathbb{E}_\nu[g^2].
$$
We will define that $\nu$ satisfies the $C$-Log Sobolev Inequality (LSI)
\cite{ledoux2001logarithmic,bakry2014analysis}, when
\begin{align}
    \mbox{Ent}_{\nu}(g^2)\leq C \mathbb{E}_\nu[|\nabla g|^2],\label{eq:CLSI}
\end{align}
for all smooth functions $g$ with $\mathbb{E}_\nu[|\nabla g|^2] < +\infty$.

The simplest example of a measure satisfying the LSI is the Gaussian measure with density $(2\pi)^{-d/2}\exp(-\frac{1}{2}|x|^2)$ with respect to the Lebesgue measure $d\bx$ on $\mathbb{R}^d$~\cite{bakry2014analysis}[Sec. 5.5]. A standard sufficient condition for $\frac{1}{Z}e^{-f(\bx)}d\bx$ to satisfy a C-LSI is strong convexity of $f$ (with $ H f\geq\rho I$) and $C=\frac{2}{\rho}$; see e.g.~\cite{bakry2014analysis}[Sec. 5.7],~\cite{ledoux2001logarithmic}. Extensions for certain non convex functions $f$ exist via Holley-Stroock perturbations \cite{bakry2014analysis}[Prop. 5.1.6], e.g. for sums of strongly convex and bounded smooth functions; see also \cite{ma2019sampling,Schlichting_mixtures} for some examples.

\section{A relative entropy bound between two SDEs}\label{sec:theory}
Consider the two stochastic processes $\bx_t$, $\bx'_t$ satisfying the general dynamics, 
\begin{align}
    &d\bx_t = \mathbf{b}(\bx_t)dt + \Sigma^{\frac{1}{2}}(\bx_t) dW_t,\\
    &d\bx'_t = \mathbf{b}'(\bx'_t)dt + \Sigma'^{\frac{1}{2}}(\bx'_t)dW'_t, \label{eq:general_sde}
\end{align}
where $W_t, \, W_t'$ are standard independent Brownian motions with $\Sigma$, $\Sigma'$ positive definite symmetric matrices, $\Sigma=\Sigma^{\frac{1}{2}}(\Sigma^{\frac{1}{2}})^T$ and similarly for $\Sigma'^{\frac{1}{2}}$. The drift vector fields $\mathbf{b}$ and $ \mathbf{b}'$ are $C^2$ and satisfy the standard dissipativity condition, e.g. $\langle \mathbf{b}(\bx) ,\bx \rangle \leq c_1-c_2 ||\bx||^2$, for some $c_1,c_2>0$, which ensures global existence and uniqueness of solutions, as well as of a unique (but in general unknown) invariant measure
~\cite{MSH02}. We thus allow for both different drifts and covariance matrices and in our setting $\mathbf{b}=-\nabla f$.
Denote with $p_t$ and $p'_t$ the densities of the laws of $\bx_t$ and $\bx'_t$, respectively. For simplicity, we will assume that $\bx_t$ and $\bx_t'$ have the same initial conditions, $\bx_0,\bx'_0\sim p_0$. When this is not satisfied, an additional, exponentially fast decaying, error term appears in the estimates below.  
The density $p_t$ satisfies the Fokker-Planck equation (see Chapter 4.1 in \cite{pavliotis14book}): 
\begin{align} \label{eq:fp_general}
    \partial_t p_t &= \nabla\cdot\left(\nabla f(\bx)p_t+\frac{1}{2}\nabla\cdot (\Sigma(\bx)p_t)\right)=:\mathcal{A}p_t,
\end{align}
and similarly for $p'_t$ with $\mathcal{A}'$.

We begin with a preliminary Lemma which uses similar machinery as the proof of Lemma 2.1 of \cite{bogachev2016distances}. 
\begin{lemma}\label{lem:p_q_fp_short}
Let 
\begin{align}
    \Phi = \frac{(\Sigma'-\Sigma)\nabla p'_t}{p'_t}-(h'-h) \label{eq:Phi}
\end{align} with $$h_i = \left(b_i-\textstyle\sum_{j=1}^d\partial_{x_j}\Sigma_{ij}\right)_{i=1}^d$$ and similarly for $h'$ (using $b'_i$ and $\Sigma'_{ij}$). It holds that
\begin{align}
    &\partial_t p_t = \mathcal{A} p_t,\;\; \partial_t p'_t = \mathcal{A} p'_t + \nabla\cdot(\Phi p'_t).
\end{align}
 
\end{lemma}
\begin{proof}
Observe that, dropping dependence on $t$ and $\bx$ for ease of notation, 
\begin{align}
    \partial_t p' =& \mathcal{A}'p' = \mathcal{A}p' + (\mathcal{A}'-\mathcal{A})p'\\
    =& -\sum_{i=1}^d\partial_i(b_i'p')+\frac{1}{2}\sum_{i,j=1}^d\partial_{ij}^2(\Sigma'_{ij}p') \\
    &+ \sum_{i=1}^d\partial_j(b_jp') - \frac{1}{2}\sum_{i,j=1}^d\partial_{ij}^2(\Sigma_{ij}p') + \mathcal{A}p'\\
    =& -\sum_{i=1}^d\left(\partial_i(b_i'p')-\partial_i(b_ip')\right) \\
    &+ \frac{1}{2}\sum_{i,j=1}^d \left( \partial_{ij}^2(\Sigma_{ij}'p')-\partial_{ij}^2(\Sigma_{ij}p') \right) + \mathcal{A}p'.
\end{align}
Using the product rule, 
\begin{align}
    &\partial_{ij}^2(\Sigma_{ij}'p') = \partial_{ij}^2(\Sigma_{ij}')p' \\
    &+ \partial_j(\Sigma_{ij}')\partial_i(p')+\partial_i(\Sigma_{ij}')\partial_j(p')+\Sigma_{ij}'\partial_{ij}^2(p'),
\end{align}
and noting that, 
\begin{align}
    &\left( \partial_{ij}^2(\Sigma_{ij}'p')-\partial_{ij}^2(\Sigma_{ij}p') \right) \\
    &= \left(\partial_{ij}^2(\Sigma_{ij}')-\partial_{ij}^2(\Sigma_{ij})\right)p' + 
    \left( \partial_j\Sigma_{ij}'-\partial_j\Sigma_{ij} \right) \partial_ip' \\
    &+ \left(\partial_i\Sigma_{ij}'-\partial_i\Sigma_{ij}\right)\partial_jp' + \left(\Sigma_{ij}'-\Sigma_{ij}\right)\partial_{ij}^2p',
\end{align}
we can rewrite, 
\begin{align}
    \mathcal{A}' p' &= \mathcal{A}p' + \sum_{i=1}^d \partial_i\left(\sum_{j=1}^d(\Sigma'_{ij}-\Sigma_{ij})\partial_j p' \right.\\
    &\left.- (b_i'-b_i)p'+\sum_{j=1}^d\left( \partial_j(\Sigma_{ij}')-\partial_j(\Sigma_{ij})\right)p' \right)\\
    &= \mathcal{A} p' + \nabla\cdot(\Phi p').
\end{align}
\end{proof}

We proceed with our main result. 
\begin{theorem}\label{thm:main}
Consider the stochastic processes $\bx_t$ and $\bx'_t$ as given in \eqref{eq:general_sde} and denote with $p_t(\bx)$, $p'_t(\bx)$ the probability densities of the processes $\bx_t$ and $\bx'_t$, respectively. 
Then,
\begin{align}
    \frac{d}{dt} KL(p_t||p'_t)\leq \frac{1}{2}
    \int p_t  \norm{\Sigma^{-1/2} \Phi}^2 d\bx 
\end{align} 
\end{theorem}
\begin{proof}
Consider the time-derivative of $KL(p_t||p_t')$. By Leibniz's rule and the chain rule:
\begin{align}
    \frac{d}{dt}KL(p_t||p_t') &= \frac{d}{dt}\int p_t \log\frac{p_t}{p_t'}d\bx\\
    &= \int \left((\partial_t p_t)\log\frac{p_t}{p_t'} + p_t\partial_t\left(\log \frac{p_t}{p_t'}\right)\right)d\bx\\
    &= \underbrace{\int (\partial_t p_t)\log\frac{p_t}{p_t'} d\bx}_{[A]} - \underbrace{\int \frac{p_t}{p_t'}\partial_t p_t' d\bx}_{[B]},
\end{align}
where we have differentiated the integrand and re-arranged it and in the last equality we have used $\int \partial_t p_t d\bx = \partial_t\int p_t d\bx = 0$, since $p_t \, dx$ is a probability measure for all $t$. 
We will rewrite term [A]. To ease the notation, we will surpress the dependence of $p_t,p'_t$ on time during the proof. Observe that, 
\begin{align}
    & [A] = \int \left( -\nabla\cdot (\mathbf{b} p) + \frac{1}{2}\nabla\cdot(\nabla\cdot (\Sigma p)) \right) \log\frac{p}{p'} d\bx \\
    &= \int (\mathbf{b} p) \cdot \nabla \log\frac{p}{p'} d\bx - \frac{1}{2}\int \nabla\cdot (\Sigma p)\cdot \nabla\log\frac{p}{p'}d\bx\\
    &= \int (\mathbf{b}p) \cdot\nabla \log\frac{p}{p'} d\bx + \frac{1}{2}\int (\Sigma p): H\left(\log\frac{p}{p'}\right)d\bx \label{eq:final_term_1},
\end{align}
where in the first equality we use the Fokker-Planck equation for $p$; in the second equality we apply integration by parts 
where we assume that the boundary term vanishes; in the third equality we apply again integration by parts to the second term for the divergence of a matrix 
and use the fact that the solution decays sufficiently fast at infinity~\cite{pavliotis14book}[Thm 4.1]. 
Note that $\nabla \log \frac{p}{p'} = \frac{p'}{p}\nabla\frac{p}{p'}$ and
$H\left(\log \frac{p}{p'}\right)_{ij} = \frac{p'}{p}\partial_{ij}^2\frac{p}{p'} - \left(\frac{p'}{p}\right)^2\partial_{i}\frac{p}{p'}\partial_{j}\frac{p}{p'}.$
Using this we can rewrite $[A]$ as, 
\begin{align}
    [A] =& \int \left( (\mathbf{b}p')\cdot\nabla\frac{p}{p'} + \frac{1}{2} (\Sigma p'):H\left(\frac{p}{p'}\right)\right) d\bx \\
    &- \frac{1}{2} \int p \left(\frac{p'}{p}\right)^2\left(\sum_{i,j=1}^d\Sigma_{ij}\partial_i\frac{p}{p'}\partial_j\frac{p}{p'}\right)d\bx. \label{eq:final_final_term_1}
\end{align}
Observe that by integration by parts and under our regularity conditions the solutions of the two Fokker-Planck equations decay sufficiently fast at infinity~\cite{pavliotis14book}[Thm 4.1], 
\begin{align}
    &\int (\mathcal{A} p') \frac{p}{p'}d\bx \\
    &= \int \left( -\nabla\cdot (\mathbf{b} p') + \frac{1}{2}\nabla\cdot(\nabla\cdot (\Sigma p')) \right) \frac{p}{p'} d\bx \\
    &= \int (\mathbf{b} p') \cdot \nabla \frac{p}{p'} d\bx + \frac{1}{2}\int (\Sigma p'): H\left(\frac{p}{p'}\right)d\bx.
\end{align}
Using this expression in \eqref{eq:final_final_term_1}
and combining with Lemma \ref{lem:p_q_fp_short} we rewrite, 
\begin{align}
    [A] &=\int \frac{p}{p'} (\mathcal{A}' p')d\bx - \int \frac{p}{p'}\nabla\cdot(\Phi p')d\bx \\
    &- \frac{1}{2} \int p \left(\frac{p'}{p}\right)^2\left(\sum_{i,j=1}^d\Sigma_{ij}\partial_i\frac{p}{p'}\partial_j\frac{p}{p'}\right)d\bx.
\end{align}
Note that the last term in the above expression can be rewritten as $-\frac{1}{2}\int p' \frac{p'}{p} \norm{\Sigma^{1/2} \nabla\left(\frac{p}{p'}\right)}^2$. 
Therefore,
\begin{align}
    &\frac{d}{dt}KL(p||p')=\\
    & \underbrace{-\int \frac{p}{p'} \nabla\cdot(\Phi p') d\bx}_{[C]}\underbrace{ -\frac{1}{2}\int p' \frac{p'}{p} \norm{\Sigma^{1/2} \nabla\left(\frac{p}{p'}\right)}^2}_{[D]}.\label{eq:KLequality}
\end{align}
where  we have replaced term [B] above with $\int\frac{p}{p'}\partial_tp'd\bx=\int\frac{p}{p'}\mathcal{A}'p'd\bx$. 
Using integration by parts and the decay of the densities as $|x| \rightarrow +\infty$,
\begin{align}
    [C] = \int p' \nabla \left(\frac{p}{p'}\right) \cdot\Phi d\bx.
\end{align}
The final step is to rewrite term [C] as
\begin{align}
    [C] &= \int p' \frac{p}{p'}\frac{p'}{p}\nabla\left(\frac{p}{p'}\right)\cdot\Phi d\bx\\
    &= \int p' \frac{p'}{p} \bigg\langle \frac{p}{p'}\Phi,\nabla\frac{p}{p'}\bigg\rangle d\bx.
\end{align}
Multiplying by $\Sigma\Sigma^{-1}$ and applying Young's inequality the integrand becomes, 
\begin{align}
    &\bigg|p' \frac{p'}{p} \bigg\langle \frac{p}{p'}\Phi,\nabla\frac{p}{p'}\bigg\rangle\bigg| \\
    &= \bigg| p' \frac{p'}{p} \bigg\langle \Sigma^{-1/2} \frac{p}{p'}\Phi,  \Sigma^{1/2} \nabla\frac{p}{p'}\bigg\rangle \bigg| \label{eq:young_bound_proof}\\
    &\leq \frac{1}{2} p' \frac{p'}{p}  \norm{\Sigma^{-1/2} \Phi \frac{p}{p'}}^2 + \frac{1}{2} p' \frac{p'}{p} \norm{\Sigma^{1/2} \nabla\frac{p}{p'}}^2\\
    &= \frac{1}{2}p  \norm{\Sigma^{-1/2} \Phi}^2 + \frac{1}{2} p' \frac{p'}{p} \norm{\Sigma^{1/2} \nabla\frac{p}{p'}}^2.
\end{align}
 so that the second term will cancel with term [D] and the result follows.
\end{proof}
We proceed with the following straightforward corollary that follows 
since $\bx_0$ and $\bx'_0$ share the same distribution. 
\begin{corollary}\label{cor:bound}
Under the same assumptions as Theorem \ref{thm:main} it holds, 
\begin{align} \label{eq:kl_simplified}
    KL(p_t||p'_t) \leq \frac{1}{2}\int_0^t \int p_s(\bx)  \norm{\Sigma^{-1/2}(\bx) \Phi(s,\bx)}^2 d\bx ds,
\end{align}
where we highlight the dependence on $(s,\bx)$ in $\Phi$. 
\end{corollary}
When $\Sigma=\Sigma'$ and do not depend on $\bx$, we have $\Phi = (\mathbf{b}(\bx)-\mathbf{b}'(\bx))$ so the bound coincides with the exact expression obtained using Girsanov densities between the laws of the processes; \cite{sanz2017gaussian}[Lemma 2.1]. We believe this bound will remain sharp also when there is dependence in $\bx$ and $\Sigma$ is different from $\Sigma'$ due to the additional terms that then arise in $\Phi$. The advantage of Theorem \ref{thm:main} is that it does not require absolute continuity between the laws of the SDEs that fails when $\Sigma\neq\Sigma'$ and we recover a more general expression by working directly with the time marginals via $p_t,p'_t$ and the Fokker-Planck equation. 


Using standard concentration of measure results from \cite{ledoux1999concentration}[Cor. 2.6] from the above statement the following result follows. 
\begin{corollary}\label{cor:conc_of_measure}
Let $F:=\log\frac{p_t}{p'_t}$ be a Lipschitz function with constant $||F||_{Lip}$ and let the measure $\nu_t(d\bx):=p_t(\bx)d\bx$ satisfy the LSI with constant $C_t$. Then we have the following concentration inequality: 
\begin{equation}
\nu_t\left(F\geq r +KL(p_t||p'_t)\right)\leq\exp\left(-\frac{ r^{2}}{C_t\left\Vert F\right\Vert_{Lip}^{2}}\right)\label{eq:concentration}.
\end{equation}
\end{corollary}
Both the Lipschitz and LSI requirements are standard for obtaining concentration bounds and challenging to relax \cite{massart2007concentration}. 
\begin{remark}\label{rem:lsi}
In the isotropic (and reversible) Langevin case with $\mathbf{b}=-\nabla f(\bx)$ and $\Sigma=\sigma^2 I$, checking the LSI for $\nu_t$ is standard. To make this point clearer, suppose $f$ is strongly convex with $Hf\geq\kappa I$, then the Gibbs measure $\frac{1}{Z}e^{- \frac{f}{2\sigma^2}}$ will satisfy the LSI with constant $C=\frac{2}{\rho}$ and $\rho=\frac{\sigma^2\kappa}{2}$. In this case, the LSI for the invariant distribution and the Markov semigroup and are equivalent~\cite{bakry1997sobolev} and $\nu_t$  will obey the LSI with constant $C_t = \frac{2}{\rho}\left(1 -e^{-\rho t} \right)+C_0 e^{- \rho t}$ with $C_0$ being the LSI constant of $\nu_0$.
\end{remark}
We stress that Remark \ref{rem:lsi} is not necessarily true in the nonreversible/anisotropic case and more assumptions are required on $\mathbf{b},\Sigma$ to apply these ideas.


\subsection{Comparison for the isotropic Langevin case}
In this section we will compare our bound is consistent with previously obtained results for the 
isotropic Langevin case discussed in Remark \ref{rem:lsi} 
and consider the following dynamics:
\begin{align}\label{eq:langevin_standard}
    &d\bx_t = -\nabla f(\bx_t)dt + \sigma dW_t.\\
    &d\bx'_t = -\nabla f'(\bx'_t)dt + \sigma' dW'_t.
\end{align}
For the remainder of the section we will denote $\bx_*,\bx'_*$ to be the minimisers of $f$ and $f'$ respectively and use following assumption.
\begin{assumption}\label{ass:1}
Suppose $f,f'$ are strongly convex with constants $\kappa,\kappa'$ respectively and $\nabla f,\nabla f'$ are Lipschitz functions with constants $L,L'$ respectively.
\end{assumption} 
Since $f,f'$ are strongly convex, then as mentioned in Remark \ref{rem:lsi} $p'_t(\bx)d\bx$ satisfies a $C'_t$-LSI and directly using $g^2=\frac{p_t}{p_t'}$ in \eqref{eq:CLSI} we get the simplified expression:
\begin{align} \label{eq:lsi_property}
    KL(p_t||p'_t) 
    &\leq C'_t \int \norm{\nabla \sqrt{\frac{p_t}{p'_t}}}^2p'_t d\bx\\
    &= \frac{C'_t}{4} \int p_t \norm{\nabla \log \frac{p_t}{p'_t}}^2d\bx.
\end{align}
 Compared to \eqref{eq:lsi_property}, Theorem \ref{thm:main} is more general 
and contains an explicit dependence on $\Sigma$ and $(\nabla f-\nabla f')$ compared to $\nabla\log \frac{p_t}{p'_t}$, hence making the interplay between noise structure and landscape of the gradient magnitudes apparent. 
We can combine \eqref{eq:lsi_property} with Theorem \ref{thm:main} to obtain the following proposition.
\begin{proposition}
\label{lem:KLbound}
Let Assumption \ref{ass:1} hold and assume $p_0(\bx)d\bx$ satisfies an LSI with constant $C_0$. Then for all $t>0$, we have 
{\small
\begin{align}
    KL(p_t||p'_t)  \leq & \frac{8L'^2}{\sigma^2} \left\{2(\frac{L^2}{L'^2}+2)(\frac{\sigma^2}{2\kappa} +1)+ ||\bx_*-\bx'_*||^2\right\}\\
    &\qquad\times\frac{(4+C_0{\sigma'^2\kappa'})}{\sigma^2{\sigma'^2\kappa'}}.
\end{align}
}
\end{proposition}
The proof can be found in Appendix \ref{app:proof_kl_bound_langevin}. 

\begin{remark}
The time uniform bound on $ KL(p_t||p'_t) $ requires both the LSI for $p'_t$ and Theorem \ref{thm:main} and is based on a Gr\"onwall inequality. The general anisotropic case is more challenging. Even when the LSI can be established, one would require controls for terms like $\sum_{j=1}^d\partial_{x_j}\Sigma_{ij}$ and
$(\Sigma'-\Sigma)\nabla \log p'_t$ in $\Phi$. For the latter one could use classical Gaussian density bounds for $p'_t$ \cite{aronson1967bounds,elworthy1994formulae,norris1991estimates,sheu1991some}. 
\end{remark}

A conclusion from the proof is that the relative entropy bound does not grow indefinitely 
and the privacy risk stagnates after a (possibly large) number of iterations. This is consistent with previous results in \cite{chourasia2021differential}. For the stationary regime and $t\rightarrow\infty$ we can compare the result of Proposition \ref{lem:KLbound} with applying the bound in \eqref{eq:lsi_property} directly to invariant densities $p_\infty,p'_\infty$. Both invariant densities are unique~\cite{pavliotis14book}[Ch. 4] and are trivial to characterise. For $\bx_t$ this is given by the Gibbs measure, $p_{\infty}(\bx) = \frac{1}{Z}e^{-\frac{f(\bx)}{2\sigma^2}}$ with $Z = \int e^{-\frac{f(\bx)}{2\sigma^2}} d\bx$ and similarly $p'_\infty(\bx)\propto e^{-\frac{f'(\bx)}{2\sigma'^2}}$ for $\bx'_t$. For the stationary case under Assumption \ref{ass:1} it is possible to work directly with \eqref{eq:lsi_property}. For the sake of comparison, we state a result similar to Proposition \ref{lem:KLbound} for the stationary case that uses only \eqref{eq:lsi_property}\footnote{It is possible here to use \eqref{eq:lsi_property} without Theorem \ref{thm:main} because $p_\infty$, $p'_\infty$ are known; see Remark \ref{rem:lsi}.}. 

\begin{proposition}\label{lem:KLbound_simpler}
Let Assumption \ref{ass:1} hold. Then 
{\small
\[ KL(p_\infty||p'_\infty)  \leq  \frac{L'^2}{2\kappa'\sigma'^6} \left\{2\left(\frac{\sigma'^2 L^2}{\sigma^2 L'^2}+2\right)\frac{\sigma^2}{2\kappa} + ||\bx_*-\bx'_*||^2\right\}
\]}
\end{proposition}

The proof can be found in Appendix \ref{app:gibbs}. For simplicity we can compare the bounds in Propositions \ref{lem:KLbound} and \ref{lem:KLbound_simpler}  for the case $\sigma=\sigma'$ and $t\rightarrow\infty$. Denote with $B_1$ and $B_2$ each bound respectively, then $\frac{B_1}{B_2}=16(4+C_0{\sigma'^2\kappa'})$, 
where we note that for $t\rightarrow\infty$ the term $\frac{\sigma^2}{2\kappa} +1$ in Proposition \ref{lem:KLbound} simplifies\footnote{$1$ is used as an upper bound for $e^{-2\kappa t}$; see Remark \ref{rem:const_infty} in the Appendix.} to $\frac{\sigma^2}{2\kappa}$. This ratio should be expected given the time varying bound in Proposition  \ref{lem:KLbound} uses Gr\"onwall's inequality and $\bx_t$ is initialised away from stationarity at $p_0$.

\section{From relative entropy to privacy risk}\label{sec:insights}
In this section we will apply the general result in Theorem \ref{thm:main} to privacy risks. We will use the notation $-\nabla f(\bx,D)$ instead of $\mathbf{b}(\bx)$, with $D=\{d_1,\ldots,d_N\}$ being a training data-set with $d_i\in\mathbb{R}^m$. In the context of privacy, we are interested to establish that a computationally powerful adversary with access to $D$ has low chance of guessing from the output of $\bx_t$ of \eqref{eq:SGD_dynamics} whether an arbitrary data-point $d'\in D$ was used in the training. In this context $\mathbf{b}'(\bx)=-\nabla f(\bx,D')$ with $D'=D\setminus\{d'\}$.

In the literature $\log\frac {p_t(\bx)}{p'_t(\bx)}$ is often referred to as the \emph{privacy} loss \cite{dwork2014algorithmic}. First note that the relative entropy bound can then be interpreted as the \emph{average} privacy loss, where the average is taken over the randomness of the parameter values as determined by the training algorithm. Furthermore, as we will show in this section, the relative entropy can be related to other ways of measuring privacy risks; specifically, the membership attack risk as well the $(\epsilon,\delta)$-differential privacy.  

\subsection{Membership attacks} 
Membership attacks often assume that the adversary can design a set $\mathcal{A}$ and determine that the training was based on $D$ if $\bx_t\in \mathcal{A}$ \cite{mahloujifar2022optimal,shokri2017membership, ye2021enhanced}). Under this framework the advantage of the adversary after $t$ training time has elapsed can be expressed using
\begin{align}
\textnormal{Adv}_t := \sup_{\mathcal{A}}|\mathbb{P}(\bx_t\in\mathcal{A}|D)-\mathbb{P}(\bx_t\in\mathcal{A}|D')|,
\end{align}
where the supremum of the difference of the true positive detection and the false positive rate is taken overall possible  $\mathcal{A}$. This corresponds to the total variation (TV) distance between $p_t(\bx)d\bx$ and $p'_t(\bx)d\bx$.
Then from the Pinsker-Csiz\'{a}r-Kullback inequality~\cite{bakry2014analysis}[Equation (5.2.2)], we have
\begin{align}
\textnormal{Adv}_t \leq \sqrt{\frac{KL(p_t||p'_t)}{2}},
\end{align}
so the bound in Theorem \ref{thm:main} directly applies. 

\subsection{$(\epsilon,\delta)$-differential privacy}
Under certain assumptions, the KL-divergence can be translated into the standard $(\epsilon,\delta)$-DP bound. The standard $(\epsilon,\delta)$-DP \cite{dwork2014algorithmic} can be rewritten using the following result, see e.g. Lemma 2.2 in \cite{hall2013differential}, Theorem 5 in \cite{minami2016differential} or Theorem 2 part 2 in \cite{abadi2016deep}. 
\begin{lemma}[From concentration inequality to $(\epsilon,\delta)$-differential privacy]\label{lem:conc_to_dp}
Consider a measure $\nu_t(d\bx)=p_t(\bx)d\bx$. Then, $(\epsilon,\delta)$-differential privacy is satisfied if $\nu_t\left(\ln\frac {p_t(\bx)}{p'_t(\bx)} \geq \epsilon\right)\leq \delta,$
for \emph{every possible} adjacent pair of datasets $D,D'$. 
\end{lemma} 
Combining the above Lemma with Corollary \ref{cor:conc_of_measure}, $(\epsilon,\delta)$-DP for \emph{fixed} datasets $D,D'$ holds if, 
\begin{align} \label{eq:eps_delta_kl}
    &r_{D,D'}+KL(p_t||p_t')\leq \epsilon_{D,D'}, \\
    &\exp\left(-\frac{ r_{D,D'}^{2}}{C_t \left\Vert F\right\Vert _{Lip}^{2}}\right)\leq \delta_{D,D'},
\end{align}
and then one can find the worst-case constants ($\epsilon,\delta$) for all datasets $D,D'$. 

\section{Anisotropic noise for a better privacy-accuracy trade-off} Suppose we are interested to estimate $\bx_*$, the minimizer of $f(\bx,D)$ and additive noise is added  in \eqref{eq:SGD_dynamics} for increasing privacy, so both SDEs $\bx_t$,$\bx'_t$ have the same diffusion matrix $\Sigma=\Sigma'$. Then
\begin{align}\label{eq:bound_intro}
   &KL(p_t||p'_t) \\
   &\leq ||\Sigma^{-1}||_2\int_0^t \int p_t(\bx) \norm{(\nabla f'(\bx)-\nabla f(\bx))}_2^2 d\bx ds. 
\end{align}
where $||\Sigma^{-1}||_2=\sigma _{\max }(\Sigma^{-1})$ represents the largest singular value of matrix $\Sigma^{-1}$, i.e. the smallest eigenvalue of $\Sigma$. This is consistent with the so called vectorized Gaussian Mechanism (GM) (see e.g. \cite{chanyaswad2018mvg} or Proposition 3 in \cite{hall2013differential}). GM adds Gaussian noise $\eta \sim \mathcal{N}(0,S)$ 
to 
an estimated parameter vector $\bx$ computed over dataset $D$. 
The advantages of adding anistropic Gaussian 
noise 
has been demonstrated in \cite{chanyaswad2018mvg}, who also shows that the privacy guarantee depends only on the singular values of $S$ similar to what we see here. 

The different choices for the covariance matrix can yield the same privacy guarantee while potentially resulting in better accuracy.
To show this we present the following (well-known) convergence result:
\begin{lemma}\label{lem:convergence}
Assume $f$ is a $\kappa$-strongly convex function. Consider the dynamics in \eqref{eq:SGD_dynamics}. It then holds, 
\begin{align}
    \frac{1}{2}\mathbb{E}[|\bx_t-\bx_*|_2^2] 
    &\leq e^{-2\kappa t} + \frac{\textnormal{Tr}(\Sigma)}{4\kappa}. 
\end{align}
\end{lemma}
The proof is contained in Appendix \ref{app:convergence}. Lemma \ref{lem:convergence} shows that at large $t$ or close to convergence the error is  approximately given by $\mathbb{E}[||\bx_t-\bx_*||_2^2] \lesssim \textnormal{Tr}(\Sigma)$. The optimal choice of $\Sigma$ could minimize privacy risk for a given accuracy or vice versa. We will formulate such an optimisation problem in Section \ref{sec:impact_aniso}. 

In the context of membership inference attacks, the result in Theorem Theorem \ref{thm:main} and Lemma \ref{lem:convergence} shows that anisotropic noise may decrease the attack success rate while maintaining better model performance. Similarly, one may obtain better differential privacy guarantees at a smaller cost to the accuracy. 


\section{Discussion on what influences the relative entropy}
We end the theoretical part of the paper with a discussion on how the relative entropy, and hence privacy risks, are affected by flatness of the loss landscape and the batch size used in SGD. 

\noindent\textbf{Flatness of the loss landscape}
Suppose now $f(x,D)$ and $f'(x,D')$ differ in one datapoint and 
let $D=\{d_1,...,d_j,...d_N\}$ and $D'=\{d_1,...,d_j',...d_N\}$
where $d_j'\neq d_j$. Then
by the mean-value theorem the term in \eqref{eq:bound_intro} can be rewritten as, 
\begin{align}
    &\norm{\nabla_{\bx} f(\bx,D)-\nabla_{\bx} f(\bx,D')}_2 \\
    &\leq ||\nabla_{d} \nabla_{\bx} f(\bx,\mathbf{z})||\cdot ||d_j-d_k||,
\end{align}
where $\mathbf{z}\in [d_j,d_k]$. The sensitivity to parameter changes and datapoint changes influences the relative entropy, with less sensitivity in either decreasing it. Heuristically, the flatness as measured through the trace of the Hessian \cite{hochreiter97flat} has been used for estimating the generalisation ability of a neural network. Note the  connection to flatness as measured by the parameter flatness: one can interpret the gradient with respect to the input data $d$ as a gradient with respect to the parameters $\bx$ (see \eg Section 3.3 in \cite{borovykh19}) to obtain a Hessian-like term which shows that the more flat the loss landscape, the smaller the KL divergence between the laws of the two processes and hence the smaller the privacy risk.

\noindent\textbf{Batch size in SGD}
Consider the case where $f(x)=\sum_{l=1}^{N} f_l(x,d_l) $ and $N$ is large.
In this regime SGD uses a minibatch $\mathcal{B}$ of size $n$ containing random samples of $\{1,2,\ldots,N\}$ and implements
\begin{align}
    x_{k+1} = x_k - \frac{N}{n}\sum_{m\in \mathcal{B}} \nabla f_m(x_k),
\end{align}
The Langevin dynamics \eqref{eq:SGD_dynamics} are used as a continuous time limiting model and that uses unbiasedness and standard variance calculations from simple random sampling to get an anisotropic matrix 
$
\Sigma(x)=\alpha_{n,N}\left(\sum_{l=1}^{N}\nabla f_l(x,d_l)\nabla f_l(x,d_l)^T-\nabla f(x)\nabla f(x)^T\right)$
with 
$\alpha_{n,N}=\frac{N^2}{n}(1-\frac{1}{N})$ when sampling with replacement and $\alpha_{n,N}=\frac{N^2}{n}(1-\frac{n}{N})$ for without replacement; 
see \cite{chaudhari18} for more details.
Our analysis is applicable and it is clear that in both cases a larger batch size \emph{increases} the KL divergence and hence the privacy risk. 

\section{Numerical results}
We conclude this work with several toy numerical examples that show the impact of anisotropic noise on the relative entropy. 
\subsection{The effect of anisotropic noise}\label{sec:impact_aniso}
\subsubsection{Optimising $\Sigma$}
For ease of notation we define $\mathbf{S}^g:=|\nabla f'(\bx)-\nabla f(\bx)|$ where the absolute value is taken element-wise. Then the term in \eqref{eq:bound_intro} can be upper-bounded by $||\Sigma^{-1/2}\mathbf{S}^g||_2^2$. To show the impact of anistropic noise, we first plot the covariance-dependent KL-divergence term $||\Sigma^{-1/2}\mathbf{S}^g||_2^2$ and the error bound given by $\textnormal{Tr}(\Sigma)$ for different $x$ and $y$ in $\Sigma^{1/2}=[[x,0];[0,y]]$. We consider $\mathbf{S}^g=[10,10]^T$ in the top of Figure \ref{fig:kl_general_sf1} and $\mathbf{S}^g=[10,1]^T$ in the bottom. One can conclude that in the case of a $\mathbf{S}^g$ where the dimensions have different magnitudes of gradient variation, anisotropic noise can preserve more accuracy at equal levels of privacy compared to isotropic noise. 
Consider for example the setting with $S_f=[10,10]$: adding $(4,3)$ noise (i.e. $x=3$, $y=3$) leads to the same accuracy 
and privacy risk as adding $(4,4)$. In the setting with $S_f=[10,1]$, one can obtain the same privacy risk at $(4,1)$, which has an almost zero accuracy loss than $(4,3)$. 

\begin{figure}[t]
    \centering
    \includegraphics[width=0.2\textwidth]{{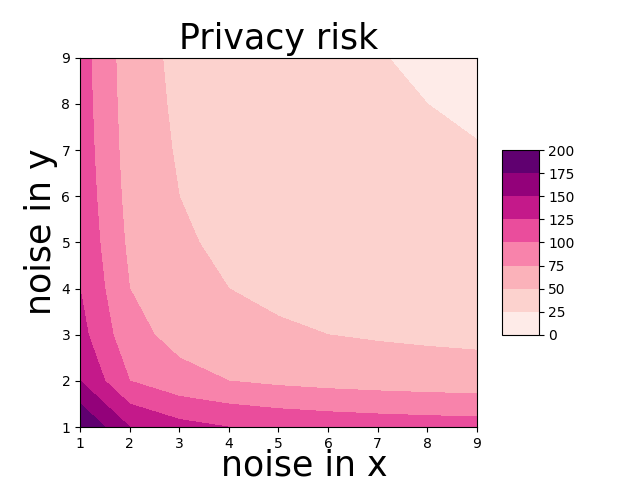}}
    \includegraphics[width=0.2\textwidth]{{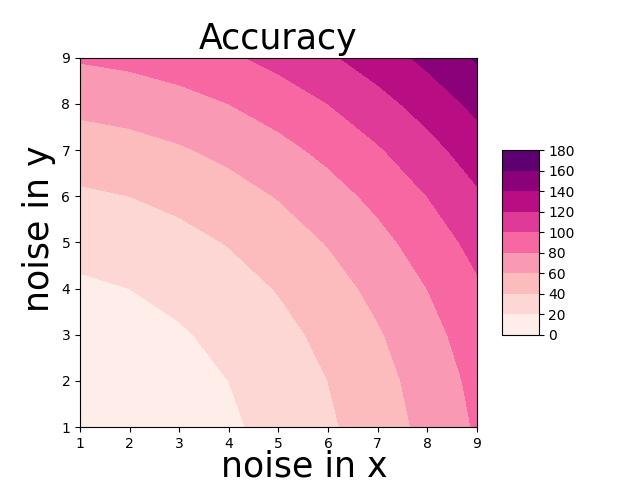}}\\
    \includegraphics[width=0.2\textwidth]{{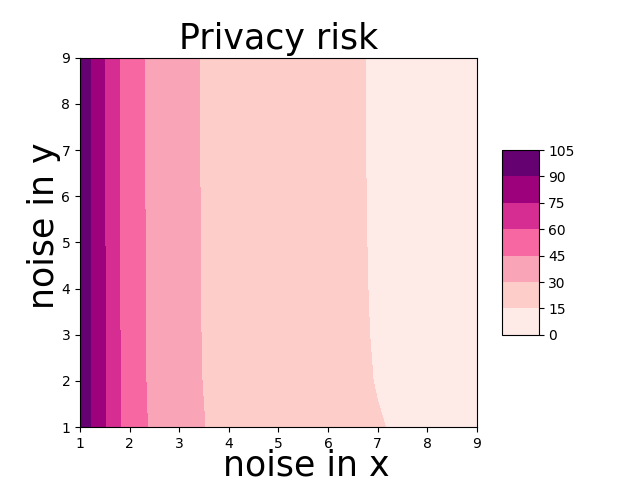}}
    \includegraphics[width=0.2\textwidth]{{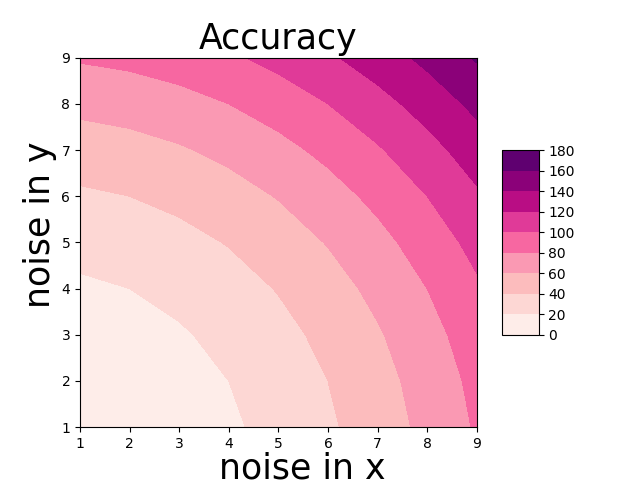}}
    \caption{(T) Setting with $\mathbf{S}^g=[10,10]^T$, (B) setting with $\mathbf{S}^g=[10,1]^T$. The accuracy is represented by the $L_2$ error, so that a low accuracy is preferable.}
    \label{fig:kl_general_sf1}
\end{figure}

We consider now the following optimization problem:
\begin{align}
& \min_{\Sigma} ||\Sigma^{-1/2}\mathbf{S}^g||_2^2, \\
& \textnormal{subject to}\quad 
 \mathbb{E}[||\bx-x^*||_2^2] = \textnormal{Tr}(\Sigma)=\zeta.
 \end{align}
If we restrict  $\Sigma^{1/2}$ to be diagonal with elements $x,y$ then we can use a generic optimiser to solve for the optimal covariance that minimises the relative entropy bound while attaining an accuracy loss of $\zeta$. 
Figure \ref{fig:opt_priv_acc} shows the optimal values for $x$ and $y$ in $\Sigma$ for different levels of $\zeta$ of the accuracy loss. In the case of $\mathbf{S}^g$ having different magnitudes in the two directions (the left-hand figure) anisotropic noise is more beneficial (since $x\neq y$), while in the case that $\mathbf{S}^g$ is the same in each dimensions (the right-hand figure) the optimal noise level is the same in both $x$ and $y$. 

\begin{figure}[h]
    \centering
    \includegraphics[width=0.2\textwidth]{{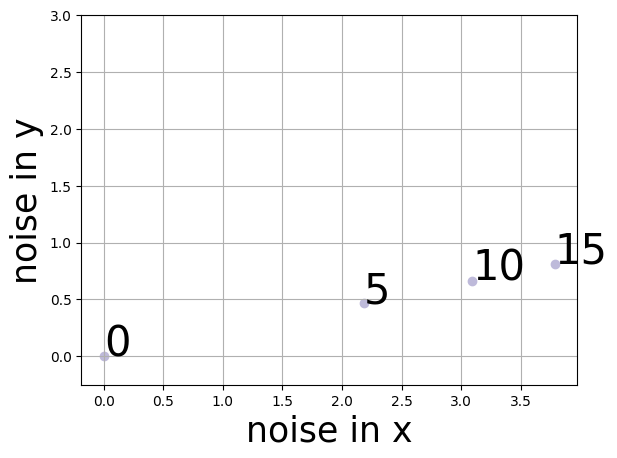}}
    \includegraphics[width=0.2\textwidth]{{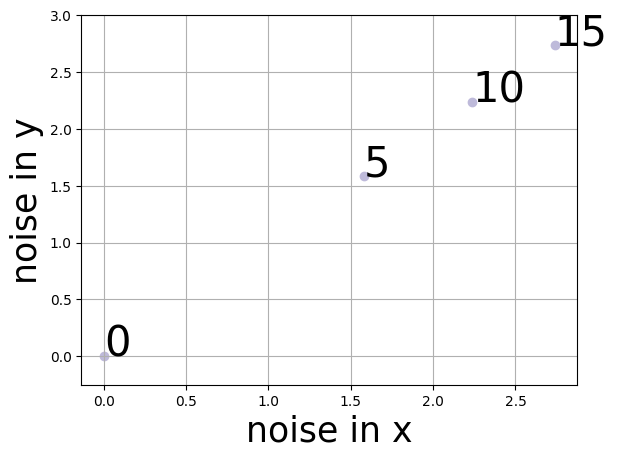}}
    \caption{Setting with $\mathbf{S}^g=[10,1]^T$ (L) and $\mathbf{S}^g=[10,10]^T$ (R). The numbers near the scatters show the $\zeta$ (loss in accuracy) value.}
    \label{fig:opt_priv_acc}
\end{figure}

\subsubsection{The quadratic Objective function} Consider a quadratic objective, $x^\top \mathbf{B}x$. It is straightforward to check that the law of the processes $x_t$ and $x'_t$ satisfy a LSI; see Appendix  \ref{sec:quad_appx} for relevant calculations for $p_t,p'_t$. In the next numerical experiment we measure the impact of the condition number of $B$ to our DP bounds. Consider again the KL-divergence and the error for different $x$ and $y$ in $\Sigma^{1/2}=[[x,0];[0,y]]$. For the quadratic the KL-divergence can be computed explicitly, see Lemma \ref{lem:kl_gaussian}. 
The accuracy is measured using the result in Lemma \ref{lem:accuracy_quad} and here we compute it at $t=100$ to ensure we are close to  the optimum. From the DP bounds one would expect that it is crucial to add more noise into the sharp directions of the difference in the objective functions to decrease the privacy risk. Using isotropic noise in all directions, the accuracy will get decreased more than necessary. Indeed, as Figure \ref{fig:quad_10_100} confirm, anisotropic noise can better maintain the privacy-accuracy tradeoff. The difference between isotropic and anisotropic noise is even more prominent for higher condition numbers. 


\begin{figure}[t]
    \centering
    \includegraphics[width=0.2\textwidth]{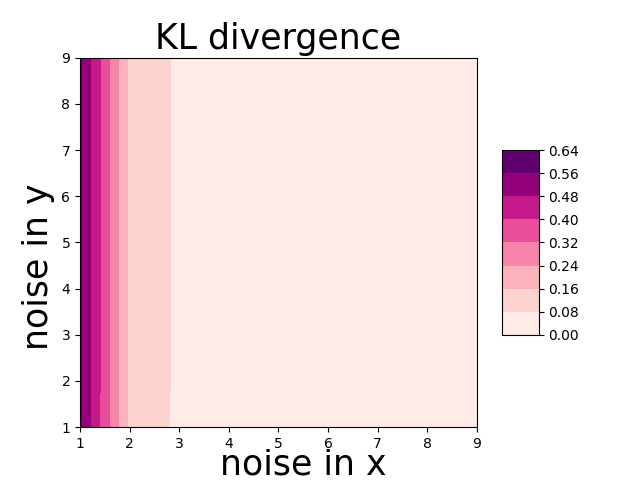}
    \includegraphics[width=0.2\textwidth]{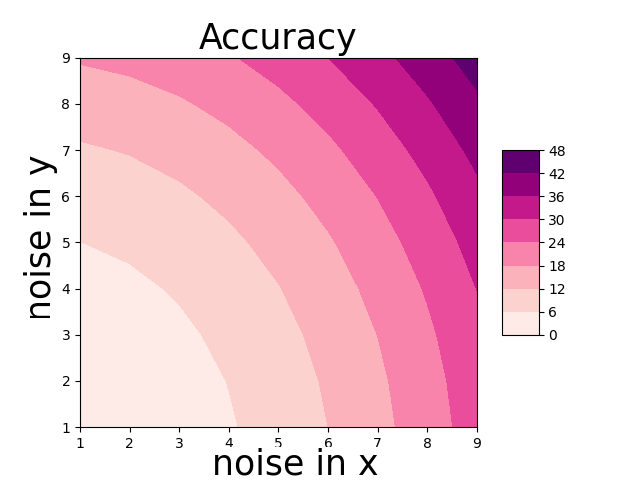}\\
    \includegraphics[width=0.2\textwidth]{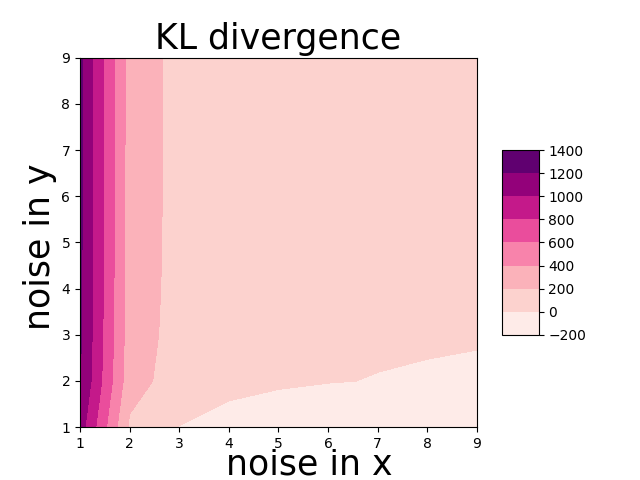}
    \includegraphics[width=0.2\textwidth]{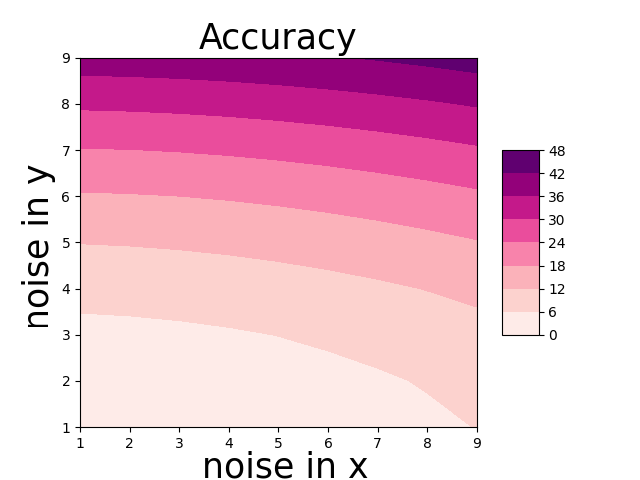}
    \caption{(T) The quadratic objective function with condition number 10; (B) the quadratic objective function with condition number 100.}\label{fig:quad_10_100}
\end{figure}


\subsection{$(\epsilon,\delta)$-DP for neural networks under anisotropic noise}\label{sec:cov_imp_num}
We consider the impact of adding anisotropic noise in the training of neural networks with different architectures. The models are trained for 10000 iterations to ensure convergence in all settings. In order to validate our theoretical conclusions numerically, we introduce a modified version of the methodology from \cite{jagielski2020auditing} to statistically estimate the differential privacy parameters of the neural network model. In particular, we use the numerical approach from Algorithm \ref{alg:1} in Appendix \ref{sec:alg_epsilon_delta} to estimate the $\delta$ for a fixed $\epsilon$. We compare the following settings:
\begin{enumerate}
    \item Gradient descent with isotropic noise per layer with dynamics $\bx^l_{t+1} = \bx^l_{t} - \eta \nabla_{\bx^l_{t}} f(\bx_{t}) + \mathcal{N}(0,\Sigma_t^l)$ where $\Sigma_t^l = \sigma^2 K_t I_{d_l\times d_l}$ where we use $d_l$ to denote the number of parameters in neural network layer $l$ and $K_t$ is set to the maximum absolute value of the gradients per iteration \emph{per layer}. Similarly for $f'$. Each parameter thus has noise added based on the largest magnitude of the nodes in its layer. 
    \item Anisotropic noise: use $\bx_{t+1}=\bx_{t}-\eta\nabla f(\bx_{t}) + \mathcal{N}(0,\Sigma_t)$ where $\Sigma_t$ is a diagonal matrix with elements $\Sigma_{(ii),t} = \sigma^2 |\nabla_{\bx^i_t}f(\bx)|$ and similarly for $f'$. Each parameter thus has noise added based on its own gradient magnitude. 
\end{enumerate}
We use a learning rate of $0.1$, we set $\sigma^2=0.01$ for FashionMNIST and $\sigma^2=0.001$ for CIFAR10 and we use a batch size of 100 for FashionMNIST and 10 for CIFAR10. The results are presented in Table \ref{tab:tab2}. For FashionMNIST we observe almost similar $\delta$-s, while the anisotropic noise is able to maintain a lower loss value at a similar $\delta$. For isotropic noise on CIFAR10, the convolutional network fails to converge even with a small $\sigma^2$, highlighting how anisotropic noise allows for easier tuning while preserving better privacy properties.

\begin{table}[h]
    \centering
    \begin{tabular}{p{3cm}||p{2cm}|p{2cm}}
    \multicolumn{3}{p{7cm}}{\textbf{Setting 1}: 1000 randomly chosen samples of FashionMNIST and a fully-connected neural network with one hidden layer of size 10.}\\\hline
    & $\delta$ for $\epsilon=0.1$ & Worst loss \\\hline\hline
    1: Isotropic per layer &  0.0015 & 0.35 \\
    2: Anisotropic & 0.0012 & 0.14 \\\hline
    \end{tabular}
    \begin{tabular}{p{3cm}||p{2cm}|p{2cm}}
    \multicolumn{3}{p{7cm}}{\textbf{Setting 2}: 10000 randomly chosen samples of FashionMNIST and a fully-connected neural network with one hidden layer of size 100.}\\\hline
    & $\delta$ for $\epsilon=0.1$ & Worst loss \\\hline\hline
    1: Isotropic per layer & 0.061  &  0.37\\
    2: Anisotropic & 0.062  & 0.26\\\hline
    \end{tabular}
    \begin{tabular}{p{3cm}||p{2cm}|p{2cm}}
    \multicolumn{3}{p{7cm}}{\textbf{Setting 3}: 1000 randomly chosen samples of CIFAR10 and a convolutional neural network\tablefootnote{The first convolutional layer is of size (3, 6, 5), followed by a (2,2) pooling, followed by the second convolutional layer of size (6,16,5), followed by a fully connected layer of size (16*5*5,120), one of (120,84) and a final layer of (84,10).}.}\\\hline
    & $\delta$ for $\epsilon=0.1$ & Worst loss \\\hline\hline
    1: Isotropic per layer & 0.30 & 3.1 \\
    2: Anisotropic & 0.00060 & 0.0062\\
    \end{tabular}
    \caption{The numerically estimated $\delta$ and the highest loss observed over all iterations for different noise covariance matrices for different datasets and architectures.}
    \label{tab:tab2}
\end{table}

\subsection{Intrinsic privacy of SGD and impact of flatness}\label{sec:num_intrinsic}
In this section we numerically estimate the privacy risk of SGD with potentially degenerate noise structures to understand better the intrinsic privacy effect of SGD and to validate the theoretically obtained impact of flatness. 
We keep the same setup as in Section \ref{sec:cov_imp_num}. The models are trained for 10000 iterations to ensure convergence for all batch sizes and a learning rate of 0.1 is used. From Table \ref{tab:tab1} we see that indeed smaller batch sizes may result in lower privacy risks.  It has been argued in \cite{dai2020large} that using smaller batch sizes facilitates convergence to flatter minima, so we provide numerical evidence to the discussion in Section \ref{sec:insights} that suggests flatter minima can lead to lower privacy risk.

\begin{table}
\centering
\begin{tabular}{l|l||l}
$\epsilon$ & Batch size & Lower bound on $\delta$\\\hline\hline
0.1 & 500 & 0.07\\
0.1 & 1000 & 0.11\\\hline
0.01 & 500 & 0.20\\
0.01 & 1000 & 0.24\\\hline
0.001 & 500 & 0.30\\
0.001 & 1000 & 0.35
\end{tabular}
\caption{Numerically estimated $\delta$ for different batch sizes for 1000 randomly chosen samples of FashionMNIST and a neural network with one hidden layer of size 10.}
    \label{tab:tab1}
\end{table}

\subsection{Membership inference attacks under anisotropic noise}
We conclude with a toy example related to membership inference attacks. As before, we take two datasets $D$ and $D'$ which differ in a single datapoint $d'$, i.e. $D'=D\backslash\{d'\}$. The success of a membership inference attack relies on the ability to distinguish between distributions defined by models trained over the two datasets. Based on \cite{ye2021enhanced} one can consider the distributions of the losses for the predictions on $d'$ for models trained on $D$ and $D'$. In Figure \ref{fig:memb_dist} we plot the histogram  of the losses (over 50 runs) on $d'$ when training with isotropic and anisotropic noise as described in Section \ref{sec:cov_imp_num} over 1000 randomly chosen FashionMNIST samples with a one-layer fully-connected network with 30 hidden nodes trained for 3000 iterations. We observe that the difference in the means of losses when trained on $D$ and $D'$ with isotropic noise is similar to the difference in the anisotropic case. This means that the models under both isotropic and anisotropic noise are equally distinguishable. However, the anisotropic case has a lower worst loss (measured as the highest loss over the train data over all runs) and hence facilitates better privacy-accuracy trade-offs.

\begin{figure}[h]
    \centering
    \includegraphics[width=0.2\textwidth]{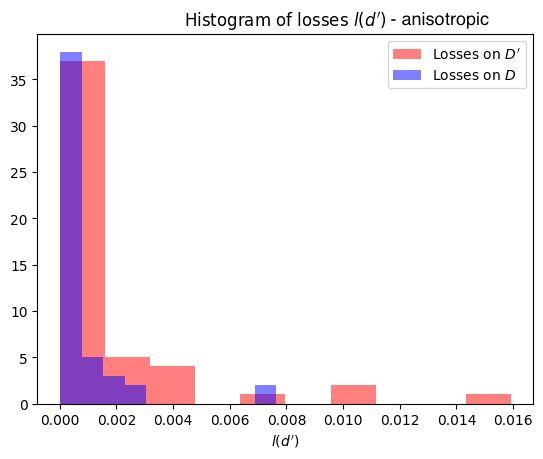}
    \includegraphics[width=0.2\textwidth]{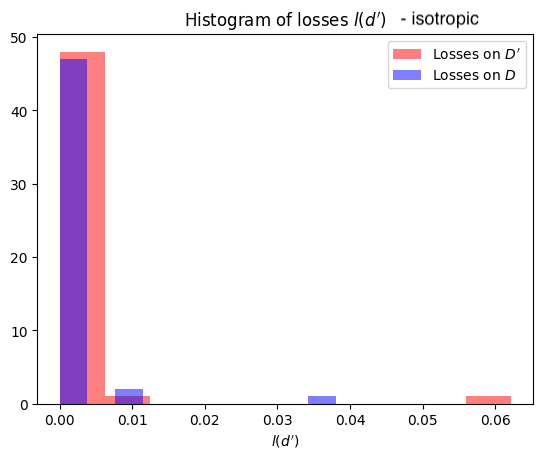}
    \caption{Cross-entropy loss on $d'$ when trained on $D$ (includes $d'$) and $D'=D\backslash\{d'\}$; (L) the setting with anisotropic noise where distance between the means of the losses when trained on $D$ and $D'$ is 0.0008 and the a worst loss of 0.43; (R) the setting with isotropic noise with distance between the means 0.0004 and a worst loss of 0.52. Anisotropic noise can lead to similar distinguishability between distributions at higher accuracy.}
    \label{fig:memb_dist}
\end{figure}

\section{Conclusion}
In this work we presented a discussion on how a very general relative entropy bound affects privacy risks for anisotropic Langevin dynamics. 
Our numerical results also considered how to choose $\Sigma$ to optimise the accuracy-privacy tradeoff and showed how anisotropic noise can lead to better accuracy at equal privacy cost. 
We speculate that even higher accuracy at equal privacy cost can be achieved by carefully tuning the magnitude of the noise, e.g. finding the directions in parameter space which are most relevant to `hiding' changes in the train dataset. 
\bibliography{main}

\begin{thebibliography}{51}
\providecommand{\natexlab}[1]{#1}
\providecommand{\url}[1]{\texttt{#1}}
\expandafter\ifx\csname urlstyle\endcsname\relax
  \providecommand{\doi}[1]{doi: #1}\else
  \providecommand{\doi}{doi: \begingroup \urlstyle{rm}\Url}\fi

\bibitem[Abadi et~al.(2016)Abadi, Chu, Goodfellow, McMahan, Mironov, Talwar,
  and Zhang]{abadi2016deep}
Abadi, M., Chu, A., Goodfellow, I., McMahan, H.~B., Mironov, I., Talwar, K.,
  and Zhang, L.
\newblock Deep learning with differential privacy.
\newblock In \emph{Proceedings of the 2016 ACM SIGSAC conference on computer
  and communications security}, pp.\  308--318, 2016.

\bibitem[Altschuler \& Talwar(2022)Altschuler and
  Talwar]{altschuler2022privacy}
Altschuler, J.~M. and Talwar, K.
\newblock Privacy of noisy stochastic gradient descent: More iterations without
  more privacy loss.
\newblock \emph{arXiv preprint arXiv:2205.13710}, 2022.

\bibitem[Aronson(1967)]{aronson1967bounds}
Aronson, D.~G.
\newblock Bounds for the fundamental solution of a parabolic equation.
\newblock \emph{Bulletin of the American Mathematical society}, 73\penalty0
  (6):\penalty0 890--896, 1967.

\bibitem[Assaraf et~al.(2018)Assaraf, Jourdain, Leli{\`e}vre, and
  Roux]{assaraf2018computation}
Assaraf, R., Jourdain, B., Leli{\`e}vre, T., and Roux, R.
\newblock Computation of sensitivities for the invariant measure of a parameter
  dependent diffusion.
\newblock \emph{Stochastics and Partial Differential Equations: Analysis and
  Computations}, 6\penalty0 (2):\penalty0 125--183, 2018.

\bibitem[Bakry(1997)]{bakry1997sobolev}
Bakry, D.
\newblock On {Sobolev} and logarithmic {Sobolev} inequalities for {M}arkov
  semigroups.
\newblock \emph{New trends in stochastic analysis (Charingworth, 1994)}, pp.\
  43--75, 1997.

\bibitem[Bakry et~al.(2014)Bakry, Gentil, Ledoux, et~al.]{bakry2014analysis}
Bakry, D., Gentil, I., Ledoux, M., et~al.
\newblock \emph{Analysis and geometry of {Markov} diffusion operators}, volume
  103.
\newblock Springer, 2014.

\bibitem[Balle \& Wang(2018)Balle and Wang]{balle18a-gaussian-mech}
Balle, B. and Wang, Y.-X.
\newblock Improving the {G}aussian mechanism for differential privacy:
  Analytical calibration and optimal denoising.
\newblock In Dy, J. and Krause, A. (eds.), \emph{Proceedings of the 35th
  International Conference on Machine Learning}, volume~80 of \emph{Proceedings
  of Machine Learning Research}, pp.\  394--403. PMLR, 10--15 Jul 2018.

\bibitem[Bogachev et~al.(2018)Bogachev, R{\"o}ckner, and
  Shaposhnikov]{bogachev2018poisson}
Bogachev, V., R{\"o}ckner, M., and Shaposhnikov, S.
\newblock The poisson equation and estimates for distances between stationary
  distributions of diffusions.
\newblock \emph{J. Math. Sci.(NY)}, 232\penalty0 (3):\penalty0 254--282, 2018.

\bibitem[Bogachev et~al.(2014)Bogachev, Kirillov, and
  Shaposhnikov]{bogachev2014kantorovich}
Bogachev, V.~I., Kirillov, A.~I., and Shaposhnikov, S.~V.
\newblock The kantorovich and variation distances between invariant measures of
  diffusions and nonlinear stationary {Fokker-Planck-Kolmogorov} equations.
\newblock \emph{Mathematical Notes}, 96\penalty0 (5):\penalty0 855--863, 2014.

\bibitem[Bogachev et~al.(2016)Bogachev, R{\"o}ckner, and
  Shaposhnikov]{bogachev2016distances}
Bogachev, V.~I., R{\"o}ckner, M., and Shaposhnikov, S.~V.
\newblock Distances between transition probabilities of diffusions and
  applications to nonlinear {Fokker--Planck--Kolmogorov} equations.
\newblock \emph{Journal of Functional Analysis}, 271\penalty0 (5):\penalty0
  1262--1300, 2016.

\bibitem[Borovykh et~al.(2019)Borovykh, Oosterlee, and Boht{\'e}]{borovykh19}
Borovykh, A., Oosterlee, C.~W., and Boht{\'e}, S.~M.
\newblock Generalization in fully-connected neural networks for time series
  forecasting.
\newblock \emph{Journal of Computational Science}, 36:\penalty0 101020, 2019.

\bibitem[Chanyaswad et~al.(2018)Chanyaswad, Dytso, Poor, and
  Mittal]{chanyaswad2018mvg}
Chanyaswad, T., Dytso, A., Poor, H.~V., and Mittal, P.
\newblock Mvg mechanism: Differential privacy under matrix-valued query.
\newblock In \emph{Proceedings of the 2018 ACM SIGSAC Conference on Computer
  and Communications Security}, pp.\  230--246, 2018.

\bibitem[Chaudhari \& Soatto(2018)Chaudhari and Soatto]{chaudhari18}
Chaudhari, P. and Soatto, S.
\newblock Stochastic gradient descent performs variational inference, converges
  to limit cycles for deep networks.
\newblock In \emph{2018 Information Theory and Applications Workshop (ITA)},
  pp.\  1--10. IEEE, 2018.

\bibitem[Chourasia et~al.(2021)Chourasia, Ye, and
  Shokri]{chourasia2021differential}
Chourasia, R., Ye, J., and Shokri, R.
\newblock Differential privacy dynamics of {Langevin} diffusion and noisy
  gradient descent.
\newblock \emph{Advances in Neural Information Processing Systems}, 34, 2021.

\bibitem[Dai \& Zhu(2020)Dai and Zhu]{dai2020large}
Dai, X. and Zhu, Y.
\newblock On large batch training and sharp minima: a {Fokker--Planck}
  perspective.
\newblock \emph{Journal of Statistical Theory and Practice}, 14\penalty0
  (3):\penalty0 1--31, 2020.

\bibitem[Duncan et~al.(2016)Duncan, Lelievre, and
  Pavliotis]{duncan2016variance}
Duncan, A.~B., Lelievre, T., and Pavliotis, G.
\newblock Variance reduction using nonreversible {Langevin} samplers.
\newblock \emph{Journal of statistical physics}, 163\penalty0 (3):\penalty0
  457--491, 2016.

\bibitem[Dwork et~al.(2014)Dwork, Roth, et~al.]{dwork2014algorithmic}
Dwork, C., Roth, A., et~al.
\newblock The algorithmic foundations of differential privacy.
\newblock \emph{Found. Trends Theor. Comput. Sci.}, 9\penalty0 (3-4):\penalty0
  211--407, 2014.

\bibitem[Eberle \& Zimmer(2019)Eberle and Zimmer]{eberle2019sticky}
Eberle, A. and Zimmer, R.
\newblock Sticky couplings of multidimensional diffusions with different
  drifts.
\newblock In \emph{Annales de l'Institut Henri Poincar{\'e}, Probabilit{\'e}s
  et Statistiques}, volume~55, pp.\  2370--2394. Institut Henri Poincar{\'e},
  2019.

\bibitem[Elworthy \& Li(1994)Elworthy and Li]{elworthy1994formulae}
Elworthy, K.~D. and Li, X.-M.
\newblock Formulae for the derivatives of heat semigroups.
\newblock \emph{Journal of Functional Analysis}, 125\penalty0 (1):\penalty0
  252--286, 1994.

\bibitem[Ganesh \& Talwar(2020)Ganesh and Talwar]{ganesh2020faster}
Ganesh, A. and Talwar, K.
\newblock Faster differentially private samplers via r{\'e}nyi divergence
  analysis of discretized {Langevin} mcmc.
\newblock \emph{Advances in Neural Information Processing Systems},
  33:\penalty0 7222--7233, 2020.

\bibitem[Ganesh et~al.(2022)Ganesh, Thakurta, and Upadhyay]{ganesh2022langevin}
Ganesh, A., Thakurta, A., and Upadhyay, J.
\newblock {Langevin} diffusion: An almost universal algorithm for private
  {Euclidean} (convex) optimization.
\newblock \emph{arXiv preprint arXiv:2204.01585}, 2022.

\bibitem[Gobet \& She(2016)Gobet and She]{gobet2016perturbation}
Gobet, E. and She, Q.
\newblock Perturbation of {Ornstein-Uhlenbeck} stationary distributions:
  expansion and simulation.
\newblock 2016.

\bibitem[Hall et~al.(2013)Hall, Rinaldo, and Wasserman]{hall2013differential}
Hall, R., Rinaldo, A., and Wasserman, L.
\newblock Differential privacy for functions and functional data.
\newblock \emph{The Journal of Machine Learning Research}, 14\penalty0
  (1):\penalty0 703--727, 2013.

\bibitem[Hochreiter \& Schmidhuber(1997)Hochreiter and
  Schmidhuber]{hochreiter97flat}
Hochreiter, S. and Schmidhuber, J.
\newblock Flat minima.
\newblock \emph{Neural Computation}, 9\penalty0 (1):\penalty0 1--42, 1997.

\bibitem[Hyland \& Tople(2019)Hyland and Tople]{hyland2019empirical}
Hyland, S.~L. and Tople, S.
\newblock An empirical study on the intrinsic privacy of {SGD}.
\newblock \emph{arXiv preprint arXiv:1912.02919}, 2019.

\bibitem[Jagielski et~al.(2020)Jagielski, Ullman, and
  Oprea]{jagielski2020auditing}
Jagielski, M., Ullman, J., and Oprea, A.
\newblock Auditing differentially private machine learning: How private is
  private {SGD}?
\newblock \emph{Advances in Neural Information Processing Systems},
  33:\penalty0 22205--22216, 2020.

\bibitem[Karatzas \& Shreve(2012)Karatzas and Shreve]{karatzas2012brownian}
Karatzas, I. and Shreve, S.
\newblock \emph{Brownian motion and stochastic calculus}, volume 113.
\newblock Springer Science \& Business Media, 2012.

\bibitem[Ledoux(1999)]{ledoux1999concentration}
Ledoux, M.
\newblock Concentration of measure and logarithmic {S}obolev inequalities.
\newblock In \emph{Seminaire de probabilites XXXIII}, pp.\  120--216. Springer,
  1999.

\bibitem[Ledoux(2001)]{ledoux2001logarithmic}
Ledoux, M.
\newblock Logarithmic {Sobolev} inequalities for unbounded spin systems
  revisited.
\newblock In \emph{S{\'e}minaire de Probabilit{\'e}s XXXV}, pp.\  167--194.
  Springer, 2001.

\bibitem[Li et~al.(2017)Li, Tai, and E]{li-noise-sgd}
Li, Q., Tai, C., and E, W.
\newblock Stochastic modified equations and adaptive stochastic gradient
  algorithms.
\newblock In Precup, D. and Teh, Y.~W. (eds.), \emph{Proceedings of the 34th
  International Conference on Machine Learning}, volume~70 of \emph{Proceedings
  of Machine Learning Research}, pp.\  2101--2110. PMLR, 06--11 Aug 2017.

\bibitem[Ma et~al.(2019)Ma, Chen, Jin, Flammarion, and Jordan]{ma2019sampling}
Ma, Y.-A., Chen, Y., Jin, C., Flammarion, N., and Jordan, M.~I.
\newblock Sampling can be faster than optimization.
\newblock \emph{Proceedings of the National Academy of Sciences}, 116\penalty0
  (42):\penalty0 20881--20885, 2019.

\bibitem[Mahloujifar et~al.(2022)Mahloujifar, Sablayrolles, Cormode, and
  Jha]{mahloujifar2022optimal}
Mahloujifar, S., Sablayrolles, A., Cormode, G., and Jha, S.
\newblock Optimal membership inference bounds for adaptive composition of
  sampled {Gaussian} mechanisms.
\newblock \emph{arXiv preprint arXiv:2204.06106}, 2022.

\bibitem[Malrieu(2001)]{malrieu01}
Malrieu, F.
\newblock Logarithmic {S}obolev inequalities for some nonlinear {PDE}'s.
\newblock \emph{Stochastic processes and their applications}, 95\penalty0
  (1):\penalty0 109--132, 2001.

\bibitem[Manita(2015)]{manita2015estimates}
Manita, O.
\newblock Estimates for kantorovich functionals between solutions to
  {Fokker--Planck--Kolmogorov} equations with dissipative drifts.
\newblock \emph{arXiv preprint arXiv:1507.04014}, 2015.

\bibitem[Massart(2007)]{massart2007concentration}
Massart, P.
\newblock \emph{Concentration inequalities and model selection: Ecole d'Et{\'e}
  de Probabilit{\'e}s de Saint-Flour XXXIII-2003}.
\newblock Springer, 2007.

\bibitem[Mattingly et~al.(2002)Mattingly, Stuart, and Higham]{MSH02}
Mattingly, J.~C., Stuart, A.~M., and Higham, D.~J.
\newblock Ergodicity for {SDEs} and approximations: locally {Lipschitz} vector
  fields and degenerate noise.
\newblock \emph{Stoch. Proc. Appl.}, 101\penalty0 (2):\penalty0 185--232, 2002.

\bibitem[Melis et~al.(2019)Melis, Song, De~Cristofaro, and
  Shmatikov]{melis2019exploiting}
Melis, L., Song, C., De~Cristofaro, E., and Shmatikov, V.
\newblock Exploiting unintended feature leakage in collaborative learning.
\newblock In \emph{2019 IEEE Symposium on Security and Privacy (SP)}, pp.\
  691--706. IEEE, 2019.

\bibitem[Minami et~al.(2015)Minami, Arai, and Sato]{minami2015differential}
Minami, K., Arai, H., and Sato, I.
\newblock ($\epsilon,\delta$)-differential privacy of {Gibbs} posteriors.
\newblock 2015.

\bibitem[Minami et~al.(2016)Minami, Arai, Sato, and
  Nakagawa]{minami2016differential}
Minami, K., Arai, H., Sato, I., and Nakagawa, H.
\newblock Differential privacy without sensitivity.
\newblock \emph{Advances in Neural Information Processing Systems}, 29, 2016.

\bibitem[Mo et~al.(2021)Mo, Borovykh, Malekzadeh, Haddadi, and
  Demetriou]{mo2021quantifying}
Mo, F., Borovykh, A., Malekzadeh, M., Haddadi, H., and Demetriou, S.
\newblock Quantifying and localizing private information leakage from neural
  network gradients.
\newblock \emph{arXiv preprint arXiv:2105.13929}, 2021.

\bibitem[Norris \& Stroock(1991)Norris and Stroock]{norris1991estimates}
Norris, J.~R. and Stroock, D.~W.
\newblock Estimates on the fundamental solution to heat flows with uniformly
  elliptic coefficients.
\newblock \emph{Proceedings of the London Mathematical Society}, 3\penalty0
  (2):\penalty0 373--402, 1991.

\bibitem[Pavliotis(2014)]{pavliotis14book}
Pavliotis, G.~A.
\newblock \emph{Stochastic processes and applications: diffusion processes, the
  {Fokker-Planck} and {Langevin} equations}, volume~60.
\newblock Springer, 2014.

\bibitem[Sanz-Alonso \& Stuart(2017)Sanz-Alonso and Stuart]{sanz2017gaussian}
Sanz-Alonso, D. and Stuart, A.~M.
\newblock {Gaussian} approximations of small noise diffusions in
  {Kullback--Leibler} divergence.
\newblock \emph{Communications in Mathematical Sciences}, 15\penalty0
  (7):\penalty0 2087--2097, 2017.

\bibitem[Schlichting(2019)]{Schlichting_mixtures}
Schlichting, A.
\newblock {Poincaré and Log–Sobolev} inequalities for mixtures.
\newblock \emph{Entropy}, 21\penalty0 (1), 2019.

\bibitem[Sheu(1991)]{sheu1991some}
Sheu, S.-J.
\newblock Some estimates of the transition density of a nondegenerate diffusion
  {Markov} process.
\newblock \emph{The Annals of Probability}, pp.\  538--561, 1991.

\bibitem[Shokri et~al.(2017)Shokri, Stronati, Song, and
  Shmatikov]{shokri2017membership}
Shokri, R., Stronati, M., Song, C., and Shmatikov, V.
\newblock Membership inference attacks against machine learning models.
\newblock In \emph{2017 IEEE symposium on security and privacy (SP)}, pp.\
  3--18. IEEE, 2017.

\bibitem[Smith et~al.(2018)Smith, {\'A}lvarez, Zwiessele, and
  Lawrence]{smith2018differentially}
Smith, M.~T., {\'A}lvarez, M.~A., Zwiessele, M., and Lawrence, N.~D.
\newblock Differentially private regression with {Gaussian} processes.
\newblock In \emph{International Conference on Artificial Intelligence and
  Statistics}, pp.\  1195--1203. PMLR, 2018.

\bibitem[Wibisono et~al.(2017)Wibisono, Jog, and Loh]{wibisono2017information}
Wibisono, A., Jog, V., and Loh, P.-L.
\newblock Information and estimation in {Fokker-Planck} channels.
\newblock In \emph{2017 IEEE International Symposium on Information Theory
  (ISIT)}, pp.\  2673--2677. IEEE, 2017.

\bibitem[Wu et~al.(2020)Wu, Hu, Xiong, Huan, Braverman, and Zhu]{wu2020noisy}
Wu, J., Hu, W., Xiong, H., Huan, J., Braverman, V., and Zhu, Z.
\newblock On the noisy gradient descent that generalizes as {SGD}.
\newblock In \emph{International Conference on Machine Learning}, pp.\
  10367--10376. PMLR, 2020.

\bibitem[Yang et~al.(2021)Yang, Xiang, Li, Liu, and Wang]{yang2021improved}
Yang, J., Xiang, L., Li, W., Liu, W., and Wang, X.
\newblock Improved matrix {Gaussian} mechanism for differential privacy.
\newblock \emph{arXiv preprint arXiv:2104.14808}, 2021.

\bibitem[Ye et~al.(2021)Ye, Maddi, Murakonda, and Shokri]{ye2021enhanced}
Ye, J., Maddi, A., Murakonda, S.~K., and Shokri, R.
\newblock Enhanced membership inference attacks against machine learning
  models.
\newblock \emph{arXiv preprint arXiv:2111.09679}, 2021.

\end{thebibliography}
\bibliographystyle{icml2023}

\appendix
\onecolumn

\section{Proofs and auxiliary results}
\subsection{Proof of Proposition \ref{lem:KLbound}}\label{app:proof_kl_bound_langevin}
For convenience we repeat the proposition statement here.

\begin{proposition}
Let Assumption \ref{ass:1} hold. Then for all $t\geq0$, 
\[ KL(p_t||p'_t)   \leq \frac{8L'^2}{\sigma^2} \left\{2(\frac{L^2}{L'^2}+2)(\frac{\sigma^2}{2\kappa} +1)+ ||\bx_*-\bx'_*||^2\right\}\frac{(4+C_0{\sigma'^2\kappa'})}{\sigma^2{\sigma'^2\kappa'}} \]
\end{proposition}
\begin{proof}
Following Remark \ref{rem:lsi} we know both $p_t(\bx)d\bx$ and $p'_t(\bx)d\bx$ obey LSIs with $$C_t = \frac{2}{\rho}\left(1 -e^{-\rho t} \right)+C_0 e^{- \rho t}$$ and 
\begin{align}
 C'_t = \frac{2}{\rho'}\left(1 -e^{-\rho' t} \right)+C_0 e^{- \rho' t}  \label{eq:C_t_formula}, 
\end{align}
where $\rho=\frac{\sigma^2\kappa}{2}$, $\rho'=\frac{\sigma'^2\kappa'}{2}$ and $C_0$ being the LSI constant of $\nu_0$, \cite{malrieu01}[pp. 112-113].  
We have
\begin{align}
    &\int p'_t \frac{p'_t}{p_t} \norm{\Sigma^{1/2} \nabla\frac{p_t}{p'_t}}^2d\bx
    = \sigma^2 \int p \norm{\nabla\log\frac{p_t}{p_t'}}^2 d\bx.
\end{align}
and following \eqref{eq:lsi_property} we get the following bound 
\begin{align}
   -\frac{1}{4}\sigma^2 \int p_t \norm{\nabla\log\frac{p_t}{p'_t}}^2 d\bx \leq -\frac{\sigma^2}{4 C'_t} KL(p_t||p'_t).\label{eq:KL_lsi_bound}
\end{align} 
Lets revisit \eqref{eq:young_bound_proof} in the proof of Theorem \ref{thm:main} and modify the bound as follows (we now include an additional $\sqrt{2}$ factor):
\begin{align}
    &\bigg|p_t' \frac{p_t'}{p_t} \bigg\langle \frac{p_t}{p_t'}\Phi,\nabla\frac{p_t}{p_t'}\bigg\rangle\bigg| \\
    &= \bigg| p_t' \frac{p_t'}{p_t} \bigg\langle \sqrt{2}\Sigma^{-1/2} \frac{p_t}{p_t'}\Phi,  \frac{\Sigma^{1/2}}{\sqrt{2}} \nabla\frac{p_t}{p_t'}\bigg\rangle \bigg| \\
    &= p_t  \norm{\Sigma^{-1/2} \Phi}^2 + \frac{1}{4} p_t' \frac{p_t'}{p_t} \norm{\Sigma^{1/2} \nabla\frac{p_t}{p_t'}}^2.
\end{align}
Then substituting in \eqref{eq:KLequality} we get 
\begin{align}
    \frac{d}{dt} KL(p_t||p'_t)\leq
    \int p_t  \norm{\Sigma^{-1/2} \Phi}^2 d\bx 
     - \frac{1}{4}\int p'_t \frac{p'_t}{p_t} \norm{\Sigma^{1/2} \nabla\frac{p_t}{p'_t}}^2d\bx.
\end{align} 
and using \eqref{eq:KL_lsi_bound} we get
\begin{align}
    \frac{d}{dt} KL(p_t||p'_t)\leq
    \int p_t  \norm{\Sigma^{-1/2} \Phi}^2 d\bx 
     -\frac{\sigma^2}{4 C'_t} KL(p_t||p'_t)
\end{align} 
and applying Gr\"onwall's Lemma and substituting $\Phi=\nabla f-\nabla f'$ gives
\begin{align}\label{eq:kl_simplified_langevin}
    KL(p_t||p'_t) \leq e^{r_t}KL(p_0||p'_0)+
    \frac{1}{\sigma^2} \int_0^t \int_{\mathbb{R}^d} ||\nabla f(\bx)-\nabla f'(\bx)||^2 p_s(\bx)  e^{r_t-r_s} d\bx ds, 
\end{align}
with $r_t=-\int_0^t\frac{\sigma^2}{4C'_u}du$. 
Recall $p_0=p_0'$ and from \eqref{eq:C_t_formula} we have  $C_0\leq C'_t\leq\frac{2}{\rho'}+C_0$ and hence $r_t\leq-\frac{\sigma^2{\rho'}}{4(2+C_0{\rho'})}t$ so
\begin{align}\label{eq:kl_simplified_langevin}
    KL(p_t||p'_t) \leq 
    \frac{1}{\sigma^2} \int_0^t \left[ \int_{\mathbb{R}^d} ||\nabla f(\bx)-\nabla f'(\bx)||^2 p_s(\bx)d\bx \right] e^{-\frac{\sigma^2{\rho'}}{4(2+C_0{\rho'})}(t-s)}  ds. 
\end{align}
Bounding the expectation using Lemma \ref{lem:bound_nablas} with ${M}=1$ and integrating w.r.t. time
we get
\begin{align}\label{eq:kl_simplified_langevin}
    KL(p_t||p'_t) &\leq 
    \frac{1}{\sigma^2}\left[2L'^2 \left\{2(\frac{L^2}{L'^2}+2)(\frac{\sigma^2}{2\kappa} +e^{-2\kappa t})+ ||\bx_*-\bx'_*||^2\right\}\right]\frac{4(2+C_0{\rho'})}{\sigma^2{\rho'}}(1- e^{-\frac{\sigma^2{\rho'}}{4(2+C_0{\rho'})}t}) \\
    &  \leq \frac{8L'^2}{\sigma^2} \left\{2(\frac{L^2}{L'^2}+2)(\frac{\sigma^2}{2\kappa} +1)+ ||\bx_*-\bx'_*||^2\right\}\frac{(4+C_0{\sigma'^2\kappa'})}{\sigma^2{\sigma'^2\kappa'}} ,
\end{align}
where in the last step we upper bounded the terms adding or subtracting  exponential  functions of time and substituted $\rho'=\frac{\sigma'^2\kappa'}{2}$ and the result follows.

 \end{proof}

\begin{remark}\label{rem:const_infty}
Note that at $t\rightarrow\infty$ from the second to last inequality we get a slightly sharper bound \[KL(p_t||p'_t)   \leq \frac{8L'^2}{\sigma^2}\left\{2(\frac{L^2}{L'^2}+2)(\frac{\sigma^2}{2\kappa})+ ||\bx_*-\bx'_*||^2\right\}\frac{(4+C_0{\sigma'^2\kappa'})}{\sigma^2{\sigma'^2\kappa'}} \].
\end{remark}

\subsection{Proof or Proposition \ref{lem:KLbound_simpler}}\label{app:gibbs}



For convenience we repeat again the proposition statement.

\begin{proposition}\label{prop:asym_general}
Let Assumption \ref{ass:1} holds. Consider
\begin{align}
    &p_\infty(\bx)=\frac{1}{Z}e^{-\frac{1}{2\sigma^2}f(\bx)}, \quad
    p_\infty'(\bx)=\frac{1}{Z'}e^{-\frac{1}{2\sigma^2}f'(\bx)},
\end{align}
with normalization constant $  Z = \int_{\mathbb{R}^d} e^{-\frac{1}{2\sigma^2}f(\bx)}d\bx,$ and similarly for $Z'$. 
Then,
\begin{align}
    KL(p_\infty||p_\infty') \leq  \frac{L'^2}{2\kappa'\sigma'^6} \left\{(\frac{\sigma'^2 L^2}{\sigma^2 L'^2}+2)\frac{\sigma^2}{2\kappa} + ||\bx_*-\bx'_*||^2\right\}
\end{align}
\end{proposition}
\begin{proof}
 $p_\infty(\bx) d\bx$ satisfies a LSI with constant $C=\frac{4}{\sigma^2\kappa}$ and similarly for $p_\infty'(\bx) d\bx$ with $C'=\frac{4}{\sigma'^2\kappa'}$.
Since the second LSI holds (with $C'$) from \eqref{eq:lsi_property} we get 
\begin{align} 
    KL(p_\infty||p_\infty') 
    \leq C' \int \norm{\nabla \sqrt{\frac{p_\infty}{p_\infty'}}}^2p' d\bx
\end{align}
Furthermore, we have
\begin{align} \label{eq:lsi_prior}
    \norm{\nabla \sqrt{\frac{p_\infty}{p_\infty'}}}^2 = \norm{\frac{1}{2}\sqrt{\frac{p_\infty}{p_\infty'}}\nabla \log \frac{p_\infty}{p_\infty'}}^2   \leq \frac{p_\infty}{p_\infty'}\frac{1}{4}\norm{\nabla \log \frac{p_\infty}{p_\infty'}}^2
\end{align}
and
\begin{align}\label{eq:bound_grad_invar}
    \norm{\nabla \log \frac{p_\infty}{p_\infty'}}^2  = \norm{\nabla  \log\frac{Z'}{Z} + \frac{1}{2\sigma^2}\nabla f(\bx)-\frac{1}{2\sigma'^2}f'(\bx)}^2 \leq \frac{||\nabla f(\bx)-\frac{\sigma^2}{\sigma'^2}\nabla f'(\bx))||^2}{4\sigma^4}.
\end{align}
Putting everything together we get
\begin{align}
    KL(p_\infty||p_\infty') \leq  \frac{C'}{4\sigma^4} \int p_\infty'(\bx) \frac{p_\infty(\bx)}{p_\infty'(\bx)}||\nabla f(\bx)-\frac{\sigma^2}{\sigma'^2}\nabla f'(\bx))||^2d\bx
\end{align}
Using Lemma \ref{lem:bound_nablas} with ${M}=\frac{\sigma^2}{\sigma'^2}$ we get
\begin{align}
    KL(p_\infty||p_\infty') \leq  \frac{C'}{4\sigma^4}2(\frac{\sigma^4}{\sigma'^4})L'^2 \left\{2(\frac{\sigma'^2 L^2}{\sigma^2 L'^2}+2)\frac{\sigma^2}{2\kappa} + ||\bx_*-\bx'_*||^2\right\}
\end{align}
and the required result.
\end{proof}
We note a similar result has been derived previously in \cite{minami2015differential,minami2016differential}, but with the very restrictive assumption that $f,f'$ are Lipschitz continuous or $||\nabla f||,||\nabla f'||$ are uniformly bounded. 
\subsection{Auxiliary results} \label{app:convergence}
\begin{lemma}\label{lem:bound_nablas}
Consider the dynamics in \eqref{eq:langevin_standard} and suppose Assumption \ref{ass:1} holds. Let $\bx_*,\bx'_*$ be the minimisers of $f$ and $f'$ respectively and $M$ be a positive constant.  Then we have  
\[\int_{\mathbb{R}^d} ||\nabla f(\bx)-M\nabla f'(\bx)||^2 p_s(\bx)d\bx \leq {\xi_t}\]
with the time varying bound $\xi_t=2M^2L'^2 \left\{2(\frac{L^2}{M^2L'^2}+2)(\frac{\sigma^2}{2\kappa} +e^{-2\kappa t})+ ||\bx_*-\bx'_*||^2\right\}$. 
\end{lemma}
\begin{proof}
    Consider the following decomposition
\begin{align}
    ||\nabla f(\bx)-\nabla f(\bx_*)+\nabla f'(\bx'_*)-\nabla f'(\bx)|| ^2&\leq 
 2L^2||\bx-\bx_*||^2+2M^2L'^2||\bx-\bx'_*||^2\\
   & \leq  2(L^2+2M^2L'^2)||\bx-\bx_*||^2+2M^2L'^2||\bx_*-\bx'_*||^2,
\end{align}
where we have used the Lipschitz properties, and $\nabla f(x^*)=\nabla f(x'^*)=0$ and $||a+b||^2\leq 2(||a||^2+||b||^2)$.
Taking expectations w.r.t. $p_s(\bx)d\bx$ and applying Lemma \ref{lem:convergence_app} below gives
\[\int_{\mathbb{R}^d} ||\nabla f(\bx)-\nabla f'(\bx)||^2 p_s(\bx)d\bx \leq   \underbrace{4 (L^2+2M^2L'^2)(e^{-2\kappa t}+\frac{\sigma^2}{4\kappa}) + 2M^2L'^2||\bx_*-\bx'_*||^2}_{\xi_t}\]
and we conclude.
\end{proof}

\begin{lemma}\label{lem:convergence_app}
Assume $f$ is a $\kappa$-strongly convex function and let $\bx_*$ be the minimiser of $f$. Consider the dynamics in \eqref{eq:SGD_dynamics}. It then holds, 
\begin{align}
    \frac{1}{2}\mathbb{E}[||\bx_t-\bx_*||_2^2] 
    &\leq e^{-2\kappa t} + \frac{\textnormal{Tr}(\Sigma)}{4\kappa}. 
\end{align}
\end{lemma}
\begin{proof}
Let $V_t = \frac{1}{2}||\bx_t-\bx_*||^2_2$. Using It\^o's lemma,
\begin{align}
    dV_t &= -(\bx_t-\bx_*)^T\nabla f(\bx_t)dt + \frac{1}{2}\textnormal{Tr}(\Sigma) + (\bx_t-\bx_*)^T\Sigma^{1/2} d\mathbf{W}_t\\
    &\leq -2\kappa V_tdt+ \frac{1}{2}\textnormal{Tr}(\Sigma) + (\bx_t-\bx_*)^T\Sigma^{1/2} d\mathbf{W}_t.
\end{align}
Integrating, using Gr\"onwall's inequality, taking expected values, using the fact that the It\^o stochastic integral is a martingale the result follows. 
\end{proof}
From the above result, as one would expect, noise with a smaller trace of the product of the covariance matrices better preserves the accuracy. This leads to the privacy-accuracy tradeoff: while more noise is needed to preserve the privacy, less noise is better for the accuracy. 


\section{Auxiliary results for the quadratic objective function} \label{sec:quad_appx}
Consider the quadratic function
\begin{align}\label{eq:quadratic_obj}
    f(\bx) = ||\mathbf{B}\bx-\mathbf{b}||_2^2.
\end{align}
Observe that for this objective function the dynamics in \eqref{eq:SGD_dynamics} are given by, 
\begin{align}\label{eq:general_ou}
    d\bx_t = -\mathbf{B}^T(\mathbf{B}\bx_t-\mathbf{b})dt+\Sigma^{1/2} d\mathbf{W}_t.
\end{align}
We will at times make the following assumption
\begin{align}\label{eq:condition_reversible}
    \mathbf{B}^T\mathbf{B}\Sigma= \Sigma\mathbf{B}^T\mathbf{B},
\end{align}
which ensures that the detailed balance condition holds and, therefore, the process $\bx_t$ is reversible with respect to the Gaussian Gibbs measure, see Proposition 3 in \cite{gobet2016perturbation}) and~\cite{pavliotis14book}[Ch. 3].

The following is a well-known, result (see e.g. Section 5.6 in \cite{karatzas2012brownian}). 
\begin{proposition}[Solution and distribution for the quadratic objective]\label{prop:solution_quadratic}
Consider the process in \eqref{eq:general_ou}. The solution is given by, 
\begin{align}
    \bx_t = e^{-\mathbf{B}^T\mathbf{B}t}\left(\bx_0 +\int_0^te^{\mathbf{B}^T\mathbf{B}s}\mathbf{B}^T\mathbf{b}ds + \int_0^t e^{\mathbf{B}^T\mathbf{B}s}\bSigma d\mathbf{W}_s\right).
\end{align}
The mean and covariance are given by, \begin{align}
    m_t &= \mathbb{E}[\bx_t] = e^{-\mathbf{B}^T\mathbf{B}t}\left(\bx_0 +\int_0^te^{\mathbf{B}^T\mathbf{B}s}\mathbf{B}^T\mathbf{b}ds\right),\\
    V_{t,s} &= \mathbb{E}[(\bx_s-m_s)(\bx_t-m_t)^T] \\
    &= e^{-\mathbf{B}^T\mathbf{B}s}\left(V_{0,0} + \int_0^{t\wedge s}e^{\mathbf{B}^T\mathbf{B}u}\bSigma\bSigma^T\left(e^{\mathbf{B}^T\mathbf{B}u}\right)^Tdu\right)\left(e^{-\mathbf{B}^T\mathbf{B}t}\right)^T,
\end{align}
where $t\wedge s=\min\{t,s\}$ and we will use the shorthand notation $V_t:=V_{t,t}$.  
The conditional distribution of $\bx_t|\bx_0$ is Gaussian, 
i.e. $\bx_t|\bx_0 \sim    \mathcal{N}\left(m_t,V_{t,t}\right)$
with
\begin{align} 
    m_t = \left(I_{dN} - e^{-\mathbf{B}^T\mathbf{B}t} \right) (\mathbf{B}^T\mathbf{B})^{-1}\mathbf{B}^T\mathbf{b},
\end{align} 
and using \eqref{eq:condition_reversible}: 
\begin{align}
    V_{t} &=\frac{1}{2} \left(I_{dN}-e^{-2\mathbf{B}^T\mathbf{B}t}\right)(\mathbf{B}^T\mathbf{B})^{-1}\Sigma.
\end{align}
\end{proposition}

For the invariant measure we obtain the following. 
\begin{lemma}[Invariant measure for the quadratic objective function]\label{lem:invar_quadratic}
If the eigenvalues of $-\mathbf{B}^T\mathbf{B}$ have negative real parts (i.e. it is a \emph{Hurwitz} matrix), then the invariant measure is unique and given by the Gaussian, 
\begin{align}
    \mathcal{N}\left(\int_0^\infty e^{-\mathbf{B}^T\mathbf{B}(t-s)}\mathbf{B}^T\mathbf{b}ds , \int_0^{\infty}e^{-\mathbf{B}^T\mathbf{B}(t-u)}\Sigma\left(e^{-\mathbf{B}^T\mathbf{B}(t-u)}\right)^Tdu\right) =: \mathcal{N}(\mathbf{m}_\infty,V_{\infty})
\end{align}
The mean can be solved for explicitly, 
\begin{align}
    m_{\infty} = (\mathbf{B}^T\mathbf{B})^{-1}\mathbf{B}^T\mathbf{b}.
\end{align}
Finally, if condition \ref{eq:condition_reversible} holds,
then, 
\begin{align}
    V_{\infty} = \frac{1}{2}(\mathbf{B}^T\mathbf{B})^{-1}\Sigma.
\end{align}
\end{lemma}
\begin{proof}
The statement on the invariant distribution follows along the lines of Proposition 2 in \cite{gobet2016perturbation}. Since $-\mathbf{B}^T\mathbf{B}$ is a Hurwitz matrix, it holds $e^{-\mathbf{B}^T\mathbf{B}t}\rightarrow 0$ as $t\rightarrow\infty$. The stochastic integral $e^{-\mathbf{B}^T\mathbf{B}t}\int_0^t e^{\mathbf{B}^T\mathbf{B}s}\bSigma d\mathbf{W}_s$ is Gaussian, centered 
with covariance going to $V_\infty$ as $t\rightarrow\infty$ and thus the related Wiener stochastic integral converges weakly to $\mathcal{N}(0,V_\infty)$. Since,
\begin{align}
    e^{-\mathbf{B}^T\mathbf{B}t}\left(\bx_0+\int_0^te^{\mathbf{B}^T\mathbf{B}s}\mathbf{B}^T\mathbf{b}ds\right)\rightarrow \mathbf{m}_\infty,
\end{align}
the convergence follows. The explicit expressions for the mean of the invariant measure follow from integration,
\begin{align}
    \int_0^te^{\mathbf{B}^t\mathbf{B}s}ds = \left(e^{\mathbf{B}^T\mathbf{B}t}-I_{dN}\right)(\mathbf{B}^T\mathbf{B})^{-1},
\end{align} 
and similarly, using \eqref{eq:condition_reversible}, for the covariance, 
\begin{align}
    V_{t} &= \int_0^te^{-(\mathbf{B}^T\mathbf{B})(t-s)}\Sigma^Te^{-(\mathbf{B}^T\mathbf{B})^T(t-s)}ds=\int_0^te^{-2(\mathbf{B}^T\mathbf{B})(t-s)}ds\Sigma\Sigma^T\\
    &=\frac{1}{2} \left(I_{dN}-e^{-2\mathbf{B}^T\mathbf{B}t}\right)(\mathbf{B}^T\mathbf{B})^{-1}\Sigma\Sigma^T,
\end{align}
and finally using $\lim_{t\rightarrow\infty}e^{-\mathbf{B}^T\mathbf{B}t}=0$ in both. 
\end{proof}

We will make use of the following well-known result. 
\begin{lemma}[KL-divergence of multivariate Gaussians]\label{lem:kl_gaussian}
Let $\bx\in\mathbb{R}^d$; the Gaussian density with mean $\mu$ and covariance $\Sigma$ is then given by,
\begin{align}
    p(\bx)=\frac{1}{(2\pi)^{\frac{n}{2}}\textnormal{det}(\Sigma)^{\frac{1}{2}}}\exp\left(-\frac{1}{2}(\bx-\mu)^T\Sigma^{-1}(\bx-\mu)\right).
\end{align}
Observe also,
\begin{align}
    \mathbb{E}_p[\bx^T\bx-2\mu\bx+\mu^2] = \Sigma.
\end{align}
Suppose we have two densities $p,p'$ with $\mu$, $\mu'$ and $\Sigma$, $\Sigma'$. Then,
\begin{align}
    \textnormal{KL}(p||p') = \frac{1}{2}\ln \frac{\textnormal{det}(\Sigma')}{\textnormal{det}(\Sigma)} - \frac{d}{2} + \frac{1}{2} \textnormal{Tr}\left((\Sigma')^{-1}\Sigma\right)+\frac{1}{2}(\mu'-\mu)^T(\Sigma')^{-1}(\mu'-\mu)
\end{align}
\end{lemma}

Then, 
\begin{align}\label{eq:kl_div_nonasym_gaussian}
    &KL(p_t||p_t')= \frac{1}{2}\ln \frac{\textnormal{det}(V'_t)}{\textnormal{det}(V_t)} - \frac{d}{2} + \frac{1}{2} \textnormal{Tr}\left((V_t')^{-1}V_t\right)+\frac{1}{2}(\mathbf{\hat m}_t-\mathbf{m}_t)^T(V'_t)^{-1}(\mathbf{\hat m}_t-\mathbf{m}_t),
\end{align}
and similarly for the invariant measures. 

Consider the error to optimality as $\mathbb{E}[||\bx_t-\bx_*||_2^2]$. 
\begin{lemma}[Error for quadratic]\label{lem:accuracy_quad}
It holds,
\begin{align}
    \mathbb{E}[||\bx_t-\bx_*||_2^2] = \int_0^t\bigg|\bigg|e^{\mathbf{B}^T\mathbf{B}(u-t)}\Sigma^{1/2}\bigg|\bigg|_F^2du,
\end{align}
with $||A||_F^2=Tr(A^TA)$ for $A\in\mathbb{R}^{d\times d}$. As $t\rightarrow\infty$, it holds
\begin{align}
    \mathbb{E}[||\bx_\infty-\bx_*||_2^2] =\frac{1}{2} \textnormal{Tr}\left((\Sigma^{1/2})^T(\mathbf{B}^T\mathbf{B})^{-1}\Sigma^{1/2}\right).
\end{align}
\end{lemma}
\begin{proof}
For the Gaussian case, $\bx_* = (\mathbf{B}^T\mathbf{B})^{-1}\mathbf{B}^T\mathbf{b}$. 
Using the multidimensional It\^o isometry (see \eg (3.21) in \cite{pavliotis14book}) the result follows. 
\end{proof}

\section{Estimation of the privacy risk} \label{sec:alg_epsilon_delta}

\begin{algorithm}
\begin{algorithmic}
\STATE {\bfseries Input:} Fix an $\epsilon$.\\
   \textbf{for} $t_1=1,...,T_1$\\
   \hspace{5mm} Sample adjacent $D,D'$ and initialize $n_{t_0}=0$\\
   \hspace{5mm} \textbf{for} $t_2=1,...,T_2$:\\
    \hspace{10mm} train two models on $D,D'$ with outputs $p$ and $p'$\\
   \hspace{10mm} \textbf{for} $(\bx_1,c_1),...,(\bx_N,c_N)\in D$:\\
   \hspace{15mm}\textbf{if} $\ln \frac{p_{c_i}(\bx_i)}{p'_{c_i}(\bx_i)}>\epsilon$:\\
   \hspace{20mm}$n_{t_0}=n_{t_0}+1$\\
   \hspace{5mm} Return $\delta{t_1} = n_{t_0}/(T_1*N)$\\
   \hspace{0mm}Return $\max\{\delta_{t_1}\}_{t_1}^{T_1}$
\end{algorithmic}
\caption{Algorithm to estimate $\delta$ for a fixed $\epsilon$.}\label{alg:1}
\end{algorithm}

We use the methodology in Algorithm \ref{alg:1} to estimate $\delta$ for a fixed $\epsilon$. We use a classifier with a softmax output layer and we interpret the outputs $p_k$ as probabilities of the input belonging to a certain output class $k=1,...,K$. As it is numerically challenging to choose a correct output set for a high-dimensional parameter vector, we use the  the probability of the output assigned to the correct class instead of the distribution over the parameters themselves. 
The computed $\delta$ can be seen as a \emph{lower bound} on the true $\delta$ in $(\epsilon,\delta)$-DP.  

In the numerical experiments we set $T_1=10$ and for each $t_1\in\{1,...,T_1\}$ we generate two datasets $D,D'$ and train two models $T_2=10$ times over both $D$ and $D'$. The two models have the same random seed (both in initialisation and during batch selection in training) to ensure their randomness is the same. We compare the probabilities in the output layers for the algorithms trained over the two datasets and use $p_{c_i}$, the probability assigned to the true class $c_i$ of datapoint $i$ by the softmax output layer.

\end{document}